\documentclass[12pt]{extarticle}

\usepackage[utf8]{inputenc}
\usepackage{caption}
\oddsidemargin \evensidemargin
\usepackage{amsfonts}
\usepackage{floatrow}
\usepackage{url}
\usepackage{amssymb, amsfonts, amsthm, amsmath,mathdots}
\usepackage{xfrac}
\usepackage{enumerate}
\usepackage{verbatim}
\usepackage[all]{xy}
\usepackage[utf8]{inputenc}
\usepackage{bbm}
\usepackage{mathrsfs,mathtools}
\usepackage{listings}
\usepackage{breqn}
\usepackage{textcomp}
\usepackage{moreverb}
\usepackage{sverb}
\usepackage{bm}
\usepackage{siunitx}
\usepackage{algorithm,algpseudocode,algorithmicx}
\DeclareMathOperator*{\argmin}{arg\,min}

\renewenvironment{thebibliography}[1]{
  \begin{oldthebibliography}{#1}
   \setlength{\itemsep}{0.4em}
    \setlength{\parskip}{0em}
}
{
  \end{oldthebibliography}
}

\usepackage[dvipsnames]{xcolor}
\usepackage[colorlinks=true, allcolors=blue,linkcolor=Blue,urlcolor=black,citecolor=MidnightBlue,backref=page]{hyperref}

\usepackage{cleveref}

\usepackage{tikz,pgfplots}
\usepackage{tikz-cd}
\usetikzlibrary{decorations.markings}

\numberwithin{equation}{section}
\newtheorem{theorem}{Theorem}[section]

\newtheorem{lemma}[theorem]{Lemma}
\newtheorem{fact}[theorem]{Fact}
\newtheorem{corollary}[theorem]{Corollary}
\newtheorem{proposition}[theorem]{Proposition}

\theoremstyle{definition}
\newtheorem{conjecture}[theorem]{Conjecture}
\newtheorem{definition}[theorem]{Definition}
\newtheorem{procedure}[theorem]{Procedure}

\newtheorem{examplex}[theorem]{Example}
\newenvironment{example}
{\pushQED{\qed}\examplex}
{\popQED\endexamplex}

\newtheorem{remarkx}[theorem]{Remark}
\newenvironment{remark}
{\pushQED{\qed}\remarkx}
{\popQED\endremarkx}

\newtheoremstyle{citing}
{}
{}
{\itshape}
{}
{\bfseries}
{\textbf{.}}
{.5em}
{\thmnote{#3}}
{\theoremstyle{citing}
}

\DeclareMathOperator{\deged}{EDdegree}
\DeclareMathOperator{\squerdeg}{SEdegree}
\DeclareMathOperator{\rank}{rank}

\DeclareMathOperator{\Diag}{diag}
\DeclareMathOperator{\Sing}{Sing}
\DeclareMathOperator{\ord}{ord}

\DeclareMathOperator{\GL}{GL}
\DeclareMathOperator{\Id}{I}
\DeclareMathOperator{\vect}{vec}
\DeclareMathOperator{\Subv}{Subv}
\DeclareMathOperator{\Part}{Part}

\newcommand{\PP}{\mathbb{P}}
\renewcommand{\d}{\mathrm{d}}

\newcommand{\cM}{\mathcal{M}}
\newcommand{\cF}{\mathcal{F}}
\newcommand{\cP}{\mathcal{P}}
\newcommand{\cS}{\mathcal{S}}
\newcommand{\cE}{\mathcal{E}}
\newcommand{\cI}{\mathcal{I}}

\newcommand{\cX}{\mathcal{X}}
\newcommand{\cR}{\mathcal{R}}
\newcommand{\NN}{\mathbb{N}}
\newcommand{\CC}{\mathbb{C}}
\newcommand{\RR}{\mathbb{R}}
\newcommand{\ZZ}{\mathbb{Z}}

\newcommand{\KK}{\mathbb{K}}

\renewcommand{\Re}{\mathfrak{Re}}
\renewcommand{\Im}{\mathfrak{Im}}

\title{Geometry of Linear Neural Networks:\\Equivariance and Invariance under Permutation Groups}
\author{Kathl\'{e}n Kohn$^{\flat}$, Anna-Laura Sattelberger$^{\flat}$, and Vahid Shahverdi$^{\flat}$}
\date{}

\pgfplotsset{compat=1.18}
\begin{document}

\maketitle 
\thispagestyle{empty}

\begin{abstract}
The set of functions parameterized by a linear fully-connected neural network is a determinantal variety. We investigate the subvariety of functions that are equivariant or invariant under the action of a permutation group. Examples of such group actions are translations or $90^\circ$ rotations on images.
We describe such equivariant or invariant subvarieties as direct products of determinantal varieties, from which we deduce  their dimension, degree,  Euclidean distance degree, and their singularities.
We fully characterize invariance for arbitrary permutation groups, and equivariance for cyclic groups. We draw conclusions for the parameterization and the design of equivariant and invariant linear networks in terms of sparsity and weight-sharing properties. We prove that all invariant linear functions can be parameterized by a single linear autoencoder with a weight-sharing property imposed by the cycle decomposition of the considered permutation. The space of rank-bounded equivariant functions has several irreducible components, so it can {\em not} be parameterized by a single network---but each irreducible~component~can. Finally, we show that minimizing the squared-error loss on our invariant or equivariant networks reduces to minimizing the Euclidean distance from determinantal varieties via the Eckart--Young theorem.
\end{abstract} 

{\hypersetup{linkcolor=black}
\setcounter{tocdepth}{1}
\renewcommand{\baselinestretch}{0.72}\normalsize
\tableofcontents
\renewcommand{\baselinestretch}{1.0}\normalsize
}

\vfill

{\small
\noindent $^{\flat}$ Department of Mathematics, KTH Royal Institute of Technology, 
100~44~Stockholm, Sweden\\
\hspace*{1.4mm} {\tt kathlen@kth.se}, {\tt anna-laura.sattelberger@mis.mpg.de}, and {\tt vahidsha@kth.se} 
}

\newpage

\section{Introduction}\label{sec:intro}
Neural networks that are equivariant or invariant under the action of a group  attract high interest both in applications and in the theory of machine learning. It is important to thoroughly study their fundamental properties. While invariance is important for classifiers, equivariance typically comes into play in feature extraction tasks.
Modding out such symmetries can drastically reduce time and memory needed for the training of neural networks.
Taking carbon emissions during the training of models \cite{climateML} into account, it is important for the role of AI in the climate crisis to encounter the increasing training and hence energy costs. This is also one of the aims that the initiative {\em Green AI} \cite{GreenAI} is thriving for, to which the construction of group equi- or invariant neural networks might contribute.

The present article investigates linear neural networks. An example of them are linear encoder-decoder models: they are families of functions $\{f_\theta\}_{\theta \in \Theta}$, parameterized by a set $\Theta=\RR^{n\times r}\times \RR^{r \times n}$. For each parameter $\theta \in \Theta$, the function $f_\theta$ is a composition of linear maps
\begin{align}\label{eq:enddec}
f_\theta\colon \ \RR^n \stackrel{f_{1,\theta}}{\longrightarrow} \RR^r \stackrel{f_{2,\theta}}{\longrightarrow} \RR^n ,
\end{align}
where $r\leq n$.
One commonly visualizes $f_\theta$ as in \Cref{fig:intro}. 
\begin{figure}[h]
	\begin{tikzpicture}[scale=0.42]
	\node[shape=circle,draw=black] (A) at (-1,5) {};
	\node[shape=circle,draw=black] (B) at (-1,3.5) {};
    \node (M) at (-1,2.24) {$\vdots$};
	\node[shape=circle,draw=black] (C) at (-1,0.5) {};
	\node[shape=circle,draw=black] (D) at (-1,-1) {};
	\node[shape=circle,draw=black] (E) at (6,3.0) {};
    \node (O) at (6,2.24) {$\vdots$} ;
	\node[shape=circle,draw=black] (F) at (6,1.0) {} ;
	\node[shape=circle,draw=black] (G) at (13,5) {} ;
	\node[shape=circle,draw=black] (H) at (13,3.5) {} ;
    \node (N) at (13,2.24) {$\vdots$} ;
	\node[shape=circle,draw=black] (I) at (13,0.5) {} ;
    \node[shape=circle,draw=black] (J) at (13,-1) {} ;
	\path [] (A) edge node[left] {}  (E); 
    \path  (A) edge node[left] {} (F);
	\path [] (B) edge node[left] {}  (E); 
    \path  (B) edge node[left] {}  (F);
	\path [] (C) edge node[left] {}  (E); 
    \path  (C) edge node[left] {}  (F);
	\path  (D) edge node[left] {}  (E); 
    \path  (D)  edge node[left] {}  (F);
	\path  (E) edge node[left] {}  (G); 
    \path  (E) edge node[left] {}  (H); 
    \path  (E) edge node[left] {}  (I);
	\path  (F)  edge node[left] {}  (G);  	   	
    \path  (F) edge node[left] {}  (H);  	   	
    \path  (F) edge node[left] {}  (I);
    \path  (E) edge node[left] {}  (J);
    \path  (F) edge node[left] {}  (J);
	\node[] at (3.2,-1) {$f_{1,\theta}$}; 
    \node[] at (8.8,-1) {$f_{2,\theta}$};
\end{tikzpicture}
\caption{A fully-connected network of depth~$2$.}
\label{fig:intro}
\end{figure}
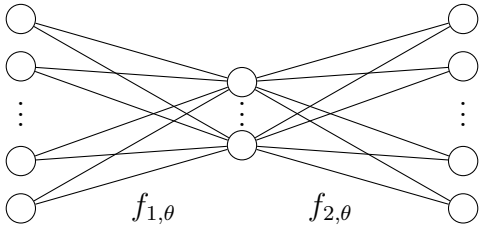

If $n=p^2$ is a square number, one can think of the input of the network as a quadratic image with $p\times p$ pixels. If $n=p^3$, the input might be a cubic $3$D scenery. In applications, one often aims to learn functions that are equi- or invariant under certain group actions, such as translations, rotations, or reflections. 

\medskip

The function space of a  linear fully-connected neural network is a determinantal variety: For natural numbers $r,m,n$, we write $\cM_{r,m\times n}$ for the subvariety of $\CC^{m\times n}$ whose points are complex $m\times n$ matrices of rank at most~$r$. In learning tasks, {\em real} matrices of rank at most~$r$ are in use; these are precisely the real-valued  points~$\cM_{r,m\times n}(\RR)$ of~$\cM_{r,m\times n}$. Reading the entries of a matrix $M$ as variables, the variety $\cM_{r,m\times n}$ is the locus of simultaneous vanishing of all $(r+1)\times (r+1)$ minors of~$M$. For a linear fully-connected neural network with input dimension $n$, output dimension $m$, and whose smallest layer has width $r$, the set of functions parameterized by it is exactly $\cM_{r,m\times n}(\RR)$.

A good understanding of the geometry of the function space of a neural network is not only mathematically interesting per se. It is useful to understand the training process of a network. For instance, it is important for understanding  the type of the critical points of the loss function. This behavior typically varies from architecture to architecture. 
In the case of linear fully-connected networks, critical points often correspond to matrices of rank even lower than~$r$, i.e., they lie in the singular locus of the determinantal variety $\cM_{r,m\times n}(\RR)$~\cite{purespurcrit}.
Investigating those points is crucial for proving the convergence of such networks to nice minima \cite{nguegnang2021convergence}. The nature of critical points is very different in the case of linear convolutional networks. Here, critical points are almost always smooth points of the function space~\cite{KMMT22,KMST23}. 

\bigskip

{\bf Main results.} In the present article, we investigate the subvarieties $\cE^G_{r,n\times n}\subset \cM_{r,n\times n}$ and $\cI^G_{r,m\times n}\subset \cM_{r,m\times n}$ of linear functions of bounded rank that are equivariant and invariant under the action of a permutation group~$G$, respectively. The group $G$ is a subgroup of the symmetric group~$\cS_n$ and acts on the input and output space $\RR^n$ (or~$\RR^m$, respectively) by permuting the entries of the input or output vector. For $m=n$, the subvarieties~$\cE^G_{r,n\times n}$ and~$\cI^G_{r,n\times n}$ encode the part of the function space of a linear autoencoder that is equi- or invariant under the action of~$G$, respectively. We provide an algebraic characterization of $\cI^G_{r,m\times n}$ for arbitrary permutation groups, and for cyclic subgroups of $\cS_n$ in the case of equivariance. Our results allow for implications on the design of equi- and invariant networks.

For invariant autoencoders, we prove a weight-sharing property on the encoder and deduce a rank constraint, i.e., a constraint on the width of the middle layer. We prove that the function space of such a constrained autoencoder is exactly~$\cI^G_{r,n\times n}$. 
In other words, linear autoencoders with our weight-sharing property on the encoder precisely parameterize invariant functions.

For equivariance, we show that the space of rank-bounded equivariant functions $\cE_{r, n \times n}^G$ typically has several irreducible components. This implies that there is \emph{no} linear neural network  that can parameterize the whole space $\cE_{r, n \times n}^G$ at once.
Every network can parameterize at most one of the 
components. This raises the natural question: {\em Which component should one choose when designing a network, and is  there one which is ``best''?}
We count the irreducible components via integer partitions of a specific form, and prove that they are direct products of determinantal varieties. We show that each of them can be parameterized by an autoencoder whose encoder and decoder have the same sparsity and weight-sharing~pattern.

\medskip

We also investigate the squared-error loss for our equi- and invariant autoencoders by relating it to Euclidean distance optimization on their function spaces. 
More concretely, we provide linear transformations that minimize the squared-error loss by minimizing the Frobenius norm on determinantal varieties which can be done using the Eckart--Young theorem. We point out that this approach cannot generally be applied to arbitrary subvarieties of~$\cM_{r,m\times n}$; it is a special feature of the varieties $\cE^G_{r,n\times n}$ and $\cI^G_{r,m\times n}$. 
To that end, we introduce the \emph{squared-error degree} as an algebraic complexity measure for minimizing the squared-error loss on a real variety, and compare it to the (generic) Euclidean distance degree.

\medskip
To showcase our results, we train linear autoencoders on the MNIST dataset, comparing architectures with and without imposed equivariance under horizontal shifts.

\medskip

Finally, we would like to highlight that our findings extend beyond linear networks by informing the construction of nonlinear equivariant architectures, as follows.
It is a common strategy to build nonlinear equivariant networks from linear equivariant layers.  These layers often involve different input and output dimensions, necessitating the use of different group representations on either side of a layer to account for how the group acts on spaces of varying dimensions. 
Instead, we suggest to consider the individual linear layers as square matrices with specific rank constraints.
That way, we can use the same group representation on both the input and output side of each layer. 
Like that, it is no longer needed
to introduce distinct group representations on either side, offering a new perspective on designing equivariant networks that maintains important structural properties---such as symmetries and dimensional constraints---within each layer, while adhering to the desired architectural constraints.

\bigskip

{\bf Related work.} 
We here give a petite sample of related references, which is by no means claimed to be exhaustive. An overview of equivariant neural networks is provided in the survey~\cite{LimNelson22}. The study of equivariance has roots in pattern recognition \cite{WoodShawe}. Group-equivariant convolutional networks were introduced by Cohen and Welling in~\cite{cohenwelling}, allowing for applications in image analysis~\cite{RotoCNN}. In \cite{Bekkers}, transitive group actions are considered. Therein, Bekkers proves that, on the level of feature maps, a linear map is equivariant if and only if it is a group convolution. Isometry- and gauge-equivariant CNNs on Riemannian manifolds were investigated in~\cite{gaugeCNN}. This geometric perspective has strong connections to physics. 

Permutation-equivariant architectures are also important in graph neural networks (GNNs), where preserving symmetry under node permutations is essential. Maron et al.~\cite{maron2018invariant} characterized all permutation-invariant and equivariant linear layers for graph data, providing a foundational framework for designing GNNs. Covariant compositional networks~\cite{hy2018predicting} were developed for molecular property prediction, incorporating physical symmetries to improve model performance. Multiresolution equivariant graph autoencoders~\cite{hy2023multiresolution} enabled hierarchical graph representation and generation. SignNet and BasisNet~\cite{lim2022sign} introduced sign- and basis-invariant architectures for spectral graph representation learning, offering robust methods for handling spectral ambiguities.

\medskip

To the best of our knowledge, our article is the first one to tackle the {\em algebro-geometric} study of equi- or invariant networks, their function~spaces,  and squared-error loss minimization on these spaces.

\bigskip

{\bf Notation.} By $\cF$, we denote the function space, and by $\Theta$ the parameter space of a neural network $F\colon \Theta\to \cF$. For a parameter $\theta\in \Theta$, we denote the function $F(\theta)$ by~$f_\theta$. The symbol~$\KK$ denotes one of the fields $\{\RR,\CC\}$. For an ideal $I$ in a polynomial ring $\KK[x_1,\ldots,x_n]$, we denote by $V(I)$  the algebraic variety $V(I)=\{x\in \KK^n \,| \, p(x)=0 \text{ for all } p\in I\}$. The symbol $\cM_{r,m\times n}$ denotes the determinantal variety in $\CC^{m\times n}$ whose points are complex $m\times n$ matrices $M$ of rank at most~$r$, and $\cM_{r,m\times n}(\RR)$ its real-valued points, i.e., real matrices $M$ of rank at most~$r$. Its subsets of equi- and invariant subspaces under the action of a group $G$ will be denoted by $\cE_{r,m\times n}^G$ and $\cI_{r,m\times n}^G$, respectively. If the letter~$r$ is dropped in the notation, this means that no rank constraint is imposed on the matrices.
For fixed $T\in \GL_n(\KK)$ and any matrix $M\in \cM_{n\times n}(\KK)$, we write shortly  $M^{\sim_T}\coloneqq T^{-1}MT $ for the respective similarity transform.
We denote the identity matrix of size $n$ by~$\Id_n$, and by~$\cS_n$ the symmetric group on the set $[n]=\{1,\ldots,n\}$. For a permutation $\sigma \in \cS_n$, $\cP(\sigma)$ denotes the partition $\{A_1,\ldots,A_k\}$ of the set~$[n]$ induced by the decomposition $\sigma=\pi_1\circ \cdots \circ \pi_k$ of $\sigma$ into pairwise disjoint cycles~$\pi_i$, where we also count trivial cycles, i.e., cycles of length~$1$.

\bigskip

{\bf Outline.} In \Cref{sec:prelim}, we present motivating examples, basics about squared-error loss minimization, and necessary preliminaries from (non-)linear algebra. In particular, we introduce the notion of \emph{squared-error degree}. \Cref{sec:inv} treats invariance under arbitrary permutation groups. \Cref{sec:equiv} characterizes equivariance of linear autoencoders under cyclic subgroups of the symmetric group. 
To demonstrate our results, we run experiments on the MNIST dataset in \Cref{sec:experiments}, comparing different architectures.
In \Cref{sec:outlook}, we present which other groups and further generalizations we are planning to tackle in future work. 

\section{Warm-up and preliminaries}\label{sec:prelim}
We start with cyclic subgroups $G$ of the symmetric group~$\cS_n$, i.e., groups of the form $G=\langle \sigma \rangle$ for some permutation $\sigma \in \cS_n$. Such groups~$G$ naturally act on the input space $\RR^n$ and the output space $\RR^n$ of the network~\eqref{eq:enddec} by permuting the entries of the in- and output vector, respectively, and on linear maps $f\colon \, \RR^n\to \RR^m$ by permuting the columns of the $m\times n$ matrix representing~$f$.

\subsection{Warm-up examples}\label{sec:warmupexps}
\subsubsection{Rotation-invariance of linear maps for \texorpdfstring{$p\times p$}{p x p} pictures}\label{sec:invrot}
Let $n=p^2$ be a square number and let $\sigma$ denote the clockwise rotation of the input $p\times p$ picture by $90$ degrees, i.e., $\sigma$ is the following permutation of the pixels $a_{ij}$:
\begin{align}\label{eq:rot}
 \sigma \colon \ \RR^{p\times p}\to \RR^{p\times p}, \qquad 
\begin{pmatrix}
 a_{11} & a_{12} & \cdots & a_{1p}\\
a_{21} & a_{22} & \cdots  & a_{2p}\\
 \vdots & \vdots & \ddots & \vdots \\
a_{p1} & a_{p2} & \cdots & a_{pp}
\end{pmatrix}
 \,\mapsto \,
\begin{pmatrix}
 a_{p1} & a_{p-1, 1} & \cdots & a_{11}\\
a_{p2} & a_{p-1,2} & \cdots  & a_{12}\\
 \vdots & \vdots & \ddots & \vdots \\
a_{pp} & a_{p, p-1} & \cdots & a_{1p}
\end{pmatrix}.
\end{align}
Since square numbers are either $0$ or $1$ modulo $4$, we will distinguish between $p$ odd and $p$~even for the identification of $\RR^{p\times p}$ with $ \RR^n$.
If $m$ is odd, we identify
\begin{align}
\RR^{p\times p} \stackrel{\cong }{\longrightarrow} \RR^n, \qquad 
  A \ =
    \begin{pmatrix}
        a_{1,1} & a_{1,2} & \cdots & a_{1,p-1} & a_{1,p}\\
        a_{2,1} & a_{2,2} & \cdots & a_{2,p-1} & a_{2,p}\\
        \vdots & \vdots  & \ddots & \vdots & \vdots \\
        a_{p-1,1} & a_{p-1,2} & \cdots & a_{p-1,p-1} & a_{p-1,p}\\
     a_{p,1} & a_{p,2} & \cdots & a_{p,p-1} & a_{p,p}
    \end{pmatrix}  \ \mapsto \ \vect(A) \, , 
\end{align}

\noindent where $ \vect(A)$ denotes $(
        a_{1,1} , a_{1,p} , a_{p,p},a_{p,1} , a_{1,2} , a_{2,p} , a_{p,p-1} , a_{p-1,1} , \ldots , a_{1,p-1} , a_{p-1,p} ,a_{p,2}, a_{2,1} ,$ $ a_{2,2}, a_{2,p-1} , a_{p-1,p-1},a_{p-1,2} , \ldots , a_{\frac{p+1}{2},\frac{p+1}{2}} 
)^\top \, .$ The intuition of the choice of vectorization can be described as ``passing from corner to corner clockwise, inwards layer by layer''. 
Under this identification, the action of~$\sigma$ is given by the $n\times n$ block matrix
{\small 
\begin{align}\label{eq:rotationodd}
    \begin{pmatrix}
    0 & 0 & 0 & 1 &  &  &  &  & & \\
    1 & 0 & 0 &0  &  &  & &  \\
        0 & 1 & 0  & 0 & &   &  &  & \\
        0 & 0 & 1 & 0 &  &  &   & && \\
       & & & & \cdots & & & & &\\
        & & & &  & 0 & 0 & 0 &1&\\
          & & & &  & 1 & 0 & 0 &0&\\
           & & & &  & 0 & 1 & 0 &0&\\
        & &  & &  & 0 & 0 & 1 &0&\\
       & & & & &  & & & & 1\\  
    \end{pmatrix}\, ,
\end{align}}

\noindent where non-filled entries are~$0$. If $m$ is even, we use the identification
\begin{align}
\RR^{p\times p} \stackrel{\cong }{\longrightarrow} \RR^n, \qquad 
    A \ = \begin{pmatrix}
        a_{1,1} & a_{1,2} & \cdots & a_{1,p-1} & a_{1,p}\\
        a_{2,1} & a_{2,2} & \cdots & a_{2,p-1} & a_{2,p}\\
         \vdots & \vdots & \ddots & \vdots & \vdots \\
        a_{p-1,1} & a_{p-1,2} & \cdots & a_{p-1,p-1} & a_{p-1,p}\\
     a_{p,1} & a_{p,2} & \cdots & a_{p,p-1} & a_{p,p}
    \end{pmatrix} \ \mapsto \ \vect(A) \, ,
    \end{align}
    where $ \vect(A)$ denotes $(a_{1,1} , a_{1,p} , a_{p,p},a_{p,1}, a_{1,2} , a_{2,p} , a_{p,p-1} , a_{p-1,1} , \ldots , a_{1,p-1} , a_{p-1,p} , 
        a_{p,2},a_{2,1}, $ $ a_{2,2}, a_{2,p-1} , a_{p-1,p-1},a_{p-1,2} , \ldots , a_{\frac{p}{2},\frac{p}{2}} ,a_{\frac{p}{2},\frac{p}{2}+1},a_{\frac{p}{2}+1,\frac{p}{2}},a_{\frac{p}{2}+1,\frac{p}{2}+1}
)^\top \,.$ Under this identification, $\sigma$~acts on $\RR^n$ by the $n\times n$ block matrix
\begin{align}\label{eq:rotationeven}
{\small \begin{pmatrix}
    0 & 0 & 0 & 1 &  &  &  &  &  \\
    1 & 0 & 0 &0  &  &  &   \\
    0 & 1 & 0  & 0 &    &  &  & \\
    0 & 0 & 1 & 0 &    &   & && \\
      & & & & \cdots & & & &\\
      & & & &  & 0 & 0 & 0 &1\\
      & & & &  & 1 & 0 & 0 &0\\
      & & & &  & 0 & 1 & 0 &0\\
      & &  & &  & 0 & 0 & 1 &0\\
    \end{pmatrix} }.
\end{align}

\noindent Invariance of~$f\colon \RR^n \to \RR^n$ under~$\sigma$ hence implies that columns $1$--$4$, $5$--$8$, $9$--$12$, and so on, of the matrix $M$  representing $f$ have to coincide. In particular, $\sigma$-invariance implies that the rank of $f$ is at most $\lceil \frac{p^2}{4} \rceil$, where~$\lceil \cdot \rceil$ denotes the ceiling function. Note that the set $\cI_{n\times n}^\sigma$ of all linear rotation-invariant maps $f\colon \RR^n \to \RR^n$ is a vector space.

\begin{remark} 
Also operations such as rotations, reflections, and shifting the rows of~$A$, all can be seen as special cases of permutations. 
\end{remark}

\begin{remark}[Design of rotation-invariant autoencoders]
 From what was found above, we deduce that for any linear encoder-decoder $f\colon \, \RR^{p\times p} \to \RR^r \to \RR^{p\times p}$ that is invariant under~$\sigma$, the number~$r$ can be chosen to be $\leq \lceil \frac{p^2}{4} \rceil$. A rank constraint with $r$ small imposes moreover that some of the blocks of four consecutive columns coincide, are zero, or linear combinations of each other.
\end{remark} 

\subsubsection{Rotation-equivariance of linear maps for
\texorpdfstring{$3\times 3$}{3 x 3} pictures}\label{sec:warmup}
Consider the set of $\RR$-linear maps $f\colon \RR^9 \to \RR^9$ of rank at most~$3.$ Every such map can be written as a composition of $\RR$-linear maps  
\begin{align} \label{eq:autoencoder939}
\RR^9 \stackrel{}{\longrightarrow} \RR^3 \stackrel{}{\longrightarrow} \RR^9.
\end{align} 
Those maps are encoded precisely by real $9\times 9$ matrices $M=(m_{ij})_{i,j=1,\ldots,9}\in \cM_{9\times 9}(\RR)$ all whose $4\times 4$ minors vanish.  The minors are homogeneous polynomials of degree~$4$ in the entries of the matrix~$M.$ We denote by $I_{3,9\times 9}\subset  \CC[\{m_{ij}\}]$ the ideal generated by those polynomials.
In more geometric terms, we are looking for the real points of the $45$-dimensional variety
\begin{align}\label{eq:detvar}
\mathcal{M}_{3,9\times 9} \,= \, V(I_{3,9\times 9}) \,\subset \, \CC^{9\times 9}.
\end{align}

\noindent Denote by $\sigma$ the clockwise rotation of a $3\times 3$ matrix by $90$ degrees, i.e.,
\begin{align}\label{eq:rotm}
 \sigma \colon \ \RR^{3\times 3}\to \RR^{3\times 3}, \qquad 
\begin{pmatrix}
 a_{11} & a_{12} & a_{13}\\
a_{21} & a_{22} & a_{23}\\
a_{31} & a_{32} &  a_{33}
\end{pmatrix}
 \mapsto 
\begin{pmatrix}
a_{31} & a_{21} & a_{11}\\
a_{32} & a_{22} & a_{12}\\
a_{33} & a_{23} &  a_{13}
\end{pmatrix},
\end{align}
and let $G=\langle \sigma \rangle$. This is a finite, cyclic subgroup of $\operatorname{O}(2)$ which preserves the $p\times p$ shape of the input matrix, which we interpret as a quadratic image with $n=9$ real pixels~$a_{ij}$.
We are interested in those maps that are equivariant under $G,$ i.e., linear maps~$f$ for which
\begin{align}\label{eq:rotinv}
\sigma \circ f \, = \, f \circ \sigma \, .
\end{align}
Again, we identify $\RR^{3\times 3}\cong\RR^9$ via 
\begin{align}
\begin{pmatrix}
a_{11} & a_{12} & a_{13}\\
a_{21} & a_{22} & a_{23}\\
a_{31} & a_{32} &  a_{33}
\end{pmatrix}\, \mapsto \,
\begin{pmatrix}
 a_{11} & a_{13} & a_{33} & a_{31} & a_{12} & a_{23} & a_{32}& a_{21} & a_{22}
\end{pmatrix}^\top.
\end{align}
Then $\sigma(A)$ is represented by the vector 
$(a_{31} \ a_{11} \ a_{13} \ a_{33} \ a_{21} \ a_{12} \ a_{23} \ a_{32}\ a_{22})^\top.$
Under this identification, the rotation is the permutation $\sigma=(1 \,4 \,3\,2)(5 \,8\,7\,6)\in \cS_9$ and is represented~by
\begin{align}\label{example:permute-99}
{\small 
\begin{pmatrix}
\begin{array}{c|c|c}
\begin{matrix}
0 & 0 & 0 & 1\\
1 & 0 & 0& 0 \\
0 & 1 & 0& 0\\ 
0& 0 & 1 & 0\\
\end{matrix} & 
\begin{matrix}
\text{\large 0}
\end{matrix} & 
\begin{matrix}
 0 
\end{matrix} \\ \hline
\begin{matrix}
\text{\large 0}
\end{matrix} & 
\begin{matrix}
 0& 0& 0 & 1\\
 1 & 0& 0& 0\\
 0 & 1 & 0& 0\\
 0& 0& 1& 0 \\
\end{matrix} & 
\begin{matrix}
0 
\end{matrix} \\ \hline
\begin{matrix}
  0
\end{matrix} & \begin{matrix}
 0 
\end{matrix} & 
\begin{matrix}
    1
\end{matrix}
\end{array}
\end{pmatrix} \, .}
\end{align}
from \Cref{eq:rotationodd}.
A map $f$ is equivariant under~$\sigma$, and hence under~$G$, if and only if its representing matrix $M$ satisfies
\begin{align}\label{eq:comm}
P_\sigma\cdot M \, =\, M \cdot P_\sigma \, .
\end{align}
We therefore aim to determine all matrices $M$ that commute with $P_\sigma$. Hence, a matrix $M$ is equivariant under $\sigma$ if and only if $M$ is similar to itself with the permutation matrix of~$\sigma$ as base change. 
Also condition~\eqref{eq:rotinv} can be expressed as the vanishing of polynomials read from \Cref{eq:comm}. Those~$81$ homogeneous binomials 
of degree~$1$ cut out the vector space~$\cE^\sigma_{9 \times 9}$. 
We see from \eqref{example:permute-99} that the matrices in $\cE_{9\times 9}^\sigma$ must be of the form
\begin{align}\label{eq:matrotequi}
\begin{pmatrix}
\begin{array}{c|c|c}
\begin{matrix}
\alpha_1 & \alpha_2 & \alpha_3 & \alpha_4\\
\alpha_4 & \alpha_1 & \alpha_2 & \alpha_3\\
\alpha_3 & \alpha_4 & \alpha_1 & \alpha_2\\
\alpha_2 & \alpha_3 & \alpha_4 & \alpha_1 
\end{matrix} &
\begin{matrix}
  \beta_1 & \beta_2 & \beta_3 & \beta_4\\  
  \beta_4 & \beta_1 & \beta_2 & \beta_3\\
  \beta_3 & \beta_4 & \beta_1 & \beta_2 \\
  \beta_2 & \beta_3 & \beta_4 & \beta_1\\
\end{matrix} &
\begin{matrix}
\varepsilon_3\\\varepsilon_3\\\varepsilon_3\\\varepsilon_3
\end{matrix} \\ \hline 
\begin{matrix}
     \gamma_1 & \gamma_2 & \gamma_3 & \gamma_4\\ 
     \gamma_4 & \gamma_1 & \gamma_2 & \gamma_3 \\
     \gamma_3 & \gamma_4 & \gamma_1 & \gamma_2 \\
     \gamma_2 & \gamma_3 & \gamma_4 & \gamma_1 \\
\end{matrix} &
\begin{matrix}
\delta_1 & \delta_2 & \delta_3 & \delta_4\\
\delta_4 & \delta_1 & \delta_2 & \delta_3 \\
\delta_3 & \delta_4 & \delta_1 & \delta_2 \\
 \delta_2 & \delta_3 & \delta_4 & \delta_1
\end{matrix} &
\begin{matrix}
    \varepsilon_4 \\ \varepsilon_4 \\ \varepsilon_4 \\ \varepsilon_4
\end{matrix}\\ \hline 
\begin{matrix}
    \varepsilon_1 & \varepsilon_1 & \varepsilon_1 & \varepsilon_1
\end{matrix}  &
\begin{matrix}
    \varepsilon_2 & \varepsilon_2 & \varepsilon_2 & \varepsilon_2
\end{matrix} &
\begin{matrix}
 \varepsilon_5
\end{matrix}
 \end{array}
\end{pmatrix} \, .
\end{align}
Hence, the dimension of the vector space  $\cE_{9\times 9}^\sigma$  is $81-60=4\cdot 4+5\cdot 1=21.$
The matrices in this vector space that lie in the function space of the autoencoder \eqref{eq:autoencoder939}
form the variety~$\cE_{3,9\times 9}^\sigma$, which is obtained as the intersection of $\cE^\sigma$ and $\cM_{3,9\times 9}$, i.e.,
\begin{align}
\cE_{3,9\times 9}^\sigma \, = \, \cM_{3,9\times 9}\cap\cE^\sigma_{9 \times 9} \, .
\end{align} 
In Theorems~\ref{thm:irreducible-dim} and \ref{thm:irreducible-dim-real}, we will see that this variety has $17$ irreducible components over~$\mathbb{C}$, and $5$ irreducible components over~$\mathbb{R}$.
The latter implies that there is \emph{no} linear neural network whose function space equals $\cE_{3,9\times 9}^\sigma$.
We will explain in \Cref{sec:parameterizingEquiv} how one can construct five autoencoders whose function spaces are the five real irreducible components of~$\cE_{3,9\times 9}^\sigma $. 
Note that also the determinant of the matrix~\eqref{eq:matrotequi}, cutting out $\cE_{8,9\times 9}^\sigma$, is reducible: over $\RR$, it has three factors; one of them factorizes over the complex numbers, resulting in a total of four factors over~$\CC$. 

\subsection{Squared-error loss \& Euclidean distance minimization}
Consider a linear network whose function space $\cF$ is contained in $\cM_{m \times n}(\RR)$.
Given training data $\mathcal{D}=\{(x_i,y_i) \, |\, i=1,\ldots,d\}\subset \RR^{n} \times \RR^{m}$ consisting of input and output pairs $(x_i, y_i)$, 
we collect all these in- and output vectors as the columns of matrices $X \in \RR^{n \times d}$ and $Y \in \RR^{m \times d}$, respectively.
The \emph{squared-error loss} with respect to the data matrices $X$ and $Y$ then is 
\begin{align} \label{eq:squaredErrorLoss}
    \cM_{m \times n}(\RR) \rightarrow \RR \, , \quad M \mapsto \Vert MX-Y \Vert_F^2 \, ,
\end{align}
where $\Vert \cdot \Vert_F$ denotes the Frobenius norm.
Given sufficiently many training data (more specifically, $d \geq n$) that are sufficiently generic, minimizing the squared-error loss on the function space $\cF \subset \cM_{m \times n}(\RR)$ is equivalent to minimizing a weighted Euclidean distance on~$\cF$:
\begin{lemma}\label{lem:squaredErrorED}
    If $\rank (XX^\top) = n$, then 
    \begin{align}\label{eq:weightedEDproblem}
        \argmin_{M \in \cF} \Vert MX-Y \Vert_F^2 \ =\ \argmin_{M \in \cF} \Vert M-U \Vert^2_{XX^\top}\, , \ \  \text{ where } \ U \coloneqq YX^\top \left(XX^\top\right)^{-1}.
    \end{align}
\end{lemma}
\noindent In the lemma, $\Vert \cdot \Vert_{XX^\top}$ denotes the norm induced by the inner product given by the positive definite matrix $XX^\top$, which is defined as 
\begin{align}
    \langle A,B \rangle_{XX^\top} \,\coloneqq\,  \big\langle\,   A (XX^\top)^{1/2}, \, B (XX^\top)^{1/2} \, \big\rangle_F .
    \end{align}
\begin{proof}[Proof of \Cref{lem:squaredErrorED}]
We start by observing that 
\begin{align}
    \langle M,U \rangle_{XX^\top}  \,=\, \mathrm{tr}\left(MXX^\top U^\top\right)
\,=\, \mathrm{tr}\left(MXY^\top\right)
\,=\, \langle MX,Y \rangle_F \, ,
\end{align}
from which we obtain that
\begin{align}\begin{split} 
    \Vert MX-Y \Vert_F^2 &\,=\, \langle MX, MX \rangle_F - 2 \langle MX,Y \rangle_F + \langle Y,Y \rangle_F \\
    &\,=\, \langle M, M \rangle_{XX^\top} - 2 \langle M,U \rangle_{XX^\top} + \langle Y,Y \rangle_F \\ 
    &\,=\, \Vert M-U\Vert^2_{XX^\top} \,+\, (\langle Y,Y \rangle_F - \langle U,U \rangle_{XX^\top}) \, .
\end{split}
\end{align}
The last term, in parentheses, is constant in the sense that it does not depend on~$M$, but only on the data matrices $X$ and $Y$. This proves the assertion.
\end{proof}

Let us first consider the case where $XX^\top$ is (close to)  a multiple of the identity matrix.
For instance,  when $x_1,\ldots,x_d$ are samples drawn from a multivariate normal distribution $\mathcal{N}(0,\sigma I_n)$, then the matrix $XX^\top$ converges to $\sigma^2 I_n$ as the number of samples $d$ tends to infinity.
In that case,  \eqref{eq:weightedEDproblem} is the problem of finding {a point in the function space $\mathcal{F}$ that is  closest to the point~$U$  with respect to the standard Euclidean (i.e., Frobenius) distance.}
We will see in \Cref{prop:EYinv} and \Cref{sec:EDopt} that, when $\mathcal{F}$ is either~$\cI^G_{r,m \times n}$ or one of the real irreducible components of~$\cE^\sigma_{r,m \times n}$, the minimum in~\eqref{eq:weightedEDproblem} can be easily found using singular value decompositions and the Eckart--Young theorem, which we recall now.
Given a matrix  $U \in \cM_{m \times n}(\RR)$, we write  $ U = V_1 \cdot \Diag(\sigma_1, \sigma_2, \ldots, \sigma_{\min(m,n)}) \cdot V_2$ for its singular value decomposition, where $\sigma_1 \geq \sigma_2 \geq \cdots \geq \sigma_{\min(m,n)}$ and $V_1$ and~$V_2$ are orthogonal matrices.
\begin{theorem}[Eckart--Young] \label{thm:EY}
    A solution to $
        \argmin_{M \in \cM_{r, m \times n}(\RR)} \Vert M - U \Vert_F^2$
    is 
    \begin{align}\label{eq:EYminimum}
    U^\ast \,=\, V_1 \cdot \Diag\left(\sigma_1, \ldots, \sigma_{r}, 0, \ldots, 0\right) \cdot V_2 \, .
\end{align}
If $\sigma_r > \sigma_{r+1}$,
then $U^\ast$ is the unique local minimum.
\end{theorem}

In the case when $XX^\top$ is not a multiple of the identity, but another matrix of full rank, the optimization problem~\eqref{eq:weightedEDproblem} can be more complicated.
\begin{example}\label{ex:FisEverything}
    When $\cF = \mathcal{M}_{r,m \times n}(\RR)$ is the function space of a fully-connected linear network, then the problem \eqref{eq:weightedEDproblem} is actually always equivalent to  distance minimization under the standard Euclidean norm.
    More concretely, since the left-multiplication by the positive definite matrix $(XX^\top)^{1/2}$ is an automorphism on $\cF$, we have
    \begin{align}
    \label{eq:FisEverything}
       \left( \argmin_{M \in \, \mathcal{M}_{r,m \times n}(\RR)} \Vert M-U \Vert^2_{XX^\top} \right) \cdot \left( XX^\top \right)^{1/2} \ =\ 
        \argmin_{M \in \, \mathcal{M}_{r,m \times n}(\RR)} \Vert M-U(XX^\top)^{1/2}  \Vert^2_F \, .
    \end{align}
    Hence, we can solve~\eqref{eq:weightedEDproblem} by solving the right-hand side of~\eqref{eq:FisEverything} by Eckart--Young.
\end{example}

However, when the multiplication by positive definite matrices does not leave the function space $\cF$ invariant, then the argument presented in \Cref{ex:FisEverything} does not 
{apply} and the optimization problem in \eqref{eq:weightedEDproblem} can  in general not be solved using Eckart--Young. 
In such a situation, in order to find a global minimum of \eqref{eq:weightedEDproblem}, one might have to compute all complex critical points of the optimization problem, then discard the non-real ones, and finally determine the minimum among the remaining real critical points.
Here, when working over the complex numbers, we do not mean to replace the real Frobenius norm  in \eqref{eq:squaredErrorLoss} by the Hermitian inner product---but rather by the \emph{algebraic extension} of the real norm squared, which is $\mathrm{tr}(M^\top M)$ also for complex matrices $M$.
This is because taking the conjugate transpose is \emph{not} an algebraic operation.
The benefit of working with the algebraic extension is that the number of complex critical points of \eqref{eq:weightedEDproblem} is the same for almost all data matrices~$X$ and~$Y$ whenever~$\mathcal{F}$ is Zariski closed. This number provides an upper bound for the number of real critical points and, more importantly, serves as a certificate that one has indeed found \emph{all} critical points. Therefore, in general, it measures the algebraic complexity of the minimization problem~\eqref{eq:weightedEDproblem} and we refer to it as the \emph{squared-error degree} of the algebraic variety~$\cF$. 
\begin{definition} \label{def:squaredErrorDegree}
 Let $\cF$ be a subvariety of $ \cM_{m \times n}(\RR)$, and $X \in \RR^{n \times n}$ and $Y \in \RR^{m \times n}$ generic matrices.   
 The \emph{squared-error degree} of $\cF$ is the number of complex critical points of the squared-error loss \eqref{eq:squaredErrorLoss} restricted to~$\cF$. We denote it by $\squerdeg(\cF)$.
\end{definition}

If $X$ is of the special form such that $XX^\top$ is a multiple of the identity matrix, then the squared-error degree specializes to the \emph{Euclidean distance degree} of~$\cF$, introduced in~\cite{EDdeg}. We denote it by $\deged(\cF)$.
By \Cref{ex:FisEverything}, we see that 
\begin{align}\label{eq:SEdegMrmn}
 \squerdeg (\cM_{r,m\times n}(\RR)) \,=\, \deged(\cM_{r,m\times n}(\RR))  \, ,
\end{align}
but in general, the squared-error degree is not equal to the Euclidean distance degree. 
\begin{example}
    \label{ex:hankel}
    Let $\cF$ be the space of $2 \times 2$ Hankel matrices $\left(\begin{smallmatrix}
        a & b \\ b & c
    \end{smallmatrix}\right)$ of rank at most one. 
    A computation in \texttt{Macaulay2}~\cite{M2} shows that 
    $\deged(\cF) = 2$ and $\squerdeg(\cF)=4$.
\end{example}
\begin{remark}
    There is also a notion of \emph{generic} Euclidean distance degree of a subvariety $\cF \subset \cM_{m \times n}(\RR)$, see \cite[Chapter~2]{MetricAlgGeo}. It is the number of complex critical points of the map  \begin{align}\label{eq:genericEDminimization}
        \cF \longrightarrow \RR, \quad M \mapsto \sum_{i=1}^m \sum_{j=1}^n \lambda_{ij}\cdot (m_{ij} - u_{ij})^2,
    \end{align} where {$U=(u_{ij})_{i,j}, \, \Lambda =(\lambda_{ij})_{i,j}\in \cM_{m \times n}(\RR)$} are generic matrices.
    Squared-error loss minimization in \eqref{eq:weightedEDproblem} is a special case of \eqref{eq:genericEDminimization}, where the positive definite matrix $XX^\top$ determines the weight matrix~$\Lambda$. Hence, the following relations hold:
\begin{align}\label{eq:degreeRelation}
        \deged(\cF) \,\leq\, \squerdeg(\cF)\, \leq \, \mathrm{genericEDdegree}(\cF) \, .
    \end{align}
    Both inequalities can be equalities or strict inequalities. 
    In \Cref{ex:hankel}, the three optimization degrees in~\eqref{eq:degreeRelation} are $2 < 4 = 4$.
    For $\cF = \cM_{1,2 \times 2}(\RR)$, they are $2 = 2 < 6$.
\end{remark}

For the case that $\mathcal{F}$ is either $\cI^G_{r,m \times n}$ or one of the real irreducible components of $\cE^\sigma_{r,m \times n}$, we will show in \Cref{prop:EYinv} and \Cref{sec:EDopt} that the squared-error loss minimization in \eqref{eq:weightedEDproblem} can be reduced to standard Euclidean distance minimization and solved explicitly using Eckart--Young, which in particular implies that $\squerdeg(\cF)  = \deged(\cF)$.
For the irreducible components of $\cE^\sigma_{r,m \times n}$, we will make use of the following crucial property of squared-error minimization.
\begin{lemma} \label{lem:directProduct}
    Let $\cF_1 \subset \cM_{m_1 \times n_1}(\RR)$ and $\cF_2 \subset \cM_{m_2 \times n_2}(\RR)$ be subvarieties. Consider their direct product $\cF \coloneqq \cF_1 \times \cF_2$. I.e., the elements of $\cF$ are block diagonal matrices $M_1 \oplus M_2$ with blocks $M_i \in \cF_i$.
    For all matrices $X \in \RR^{(n_1+n_2) \times d}$ and $Y \in \RR^{(m_1+m_2) \times d}$  with $\rank(XX^\top)=n_1+n_2$, every minimizer $M$ of the squared-error loss minimization \eqref{eq:weightedEDproblem} is of the form $M = M_1 \oplus M_2$, where $M_i$ is a minimizer of a squared-error loss minimization on $\cF_i$. 
\end{lemma}
\begin{proof}
    The minimization problem we need to solve is
\begin{align}\label{eq:directProductSEL}
        \argmin_{M = M_1 \oplus M_2 \,\in\, \cF_1 \times \cF_2} \Vert M-U \Vert^2_{XX^\top} \ =\,
        \argmin_{M = M_1 \oplus M_2 \,\in\,  \cF_1 \times \cF_2} \Vert (M-U)X \Vert^2_{F} \, .
    \end{align}
   After projecting~$U$ orthogonally onto the linear space of block diagonal matrices (with respect to the inner product $\langle \cdot, \cdot \rangle_{XX^\top}$), we may assume that $U = U_1 \oplus U_2$ has the same block diagonal form as $M$.
   Now, we let $X_1$ be the matrix that consists of the first $n_1$ rows of $X$, and let $X_2$ consist of the remaining $n_2$ rows, i.e., $X = \left( \begin{smallmatrix}
       X_1 \\ X_2
   \end{smallmatrix} \right)$.
   Then, $(M-U)X = \left( \begin{smallmatrix}
       (M_1-U_1)X_1 \\ (M_2-U_2)X_2
   \end{smallmatrix} \right)$,
   which shows that \eqref{eq:directProductSEL} is
   \begin{align*}
       \argmin_{M_1 \oplus M_2 \in \cF_1 \times \cF_2} \sum_{i=1}^2 \Vert (M_i - U_i) X_i \Vert^2_F
       \ = \,
       \argmin_{M_1 \in \cF_1} \Vert M_1 - U_1 \Vert^2_{X_1X_1^\top} \,\oplus\,
              \argmin_{M_2 \in \cF_2} \Vert M_2 - U_2 \Vert^2_{X_2X_2^\top}. 
   \end{align*}
   By \Cref{lem:squaredErrorED}, the last two minimization problems are squared-error loss minimizations on the individual factors $\cF_i$, since~$\rank(X_iX_i^\top) = n_i$.
\end{proof}

We point out that Euclidean distance minimization on a space of functions is of interest for learning besides its relation to the squared-error loss.
For example, consider the scenario where one has learned a function $f$ that is not equivariant under some given group action, but reasonably close to the space of equivariant function.
For instance, $f$ could be obtained by training some neural network with a regularized loss that favors functions that are almost equivariant. 
Then, one can try to turn $f$ into an equivariant function by finding the closest point on the space of equivariant functions to $f$.
For a linear function $f$ represented by a matrix $U$, that would mean to find a matrix  $M \in \cE^\sigma_{r, m \times n}$ that minimizes~$\Vert M-U \Vert^2_F$.

\subsection{Algebraic geometry of similarity transforms}\label{sec:simcomm}
 For natural numbers $m,n$ and $r\leq \min (m,n),$ the variety $\cM_{r,m\times n}$ of $m\times n$ matrices of rank at most~$r$ has dimension 
\begin{align}\label{eq:dimMrmn}
    \dim(\cM_{r,m\times n}) \,=\, r\cdot (m+n-r)\, ,
\end{align}
cf.~\cite[Proposition~12.2]{harris}. 
We remind our readers that the dimension of an affine variety means the Krull dimension of its coordinate ring. 
\begin{remark}[Real vs. complex]\label{rem:dimreal}
Since all the coefficients of the contributing polynomials are real, the real variety of real-valued points 
$\cM_{r,m\times n}(\mathbb{R}) = \cM_{r,m\times n}(\mathbb{C}) \cap \mathbb{R}^{m\times n}$ of $\cM_{r,m\times n}$ has the same dimension as the complex variety $\cM_{r,m\times n}$. The points of $\cM_{r,m\times n}(\RR)$ are real $m \times n$ matrices of rank at most~$r$. 
\end{remark}
As was pointed out in \cite[Example~19.10]{harris}, it is proven in \cite[Example~14.4.11]{FultonIC84} that the degree of $\cM_{r,m\times n}$ is
\begin{align}\label{eq:degMrmn}
    \deg\left( \cM_{r,m\times n} \right) \ =\, \prod_{i=0}^{n-r-1} \frac{(m+i)!\cdot i!}{(r+i)!\cdot (m-r+i)!}  \, .
\end{align}
\begin{fact}\label{lem:sing-mat}
Let $0 < r < n$.
A matrix $M\in \cM_{r,m\times n}$ is a singular point of $\cM_{r,m\times n}$ if and only if its rank is strictly smaller than~$r$, i.e., 
{$\Sing(\cM_{r,m\times n})=\cM_{r-1,m\times n}$.}
\hfill \qed
\end{fact} 

The Eckart--Young theorem can be extended to also list all critical points of minimizing the distance from $\cM_{r, m\times n}(\RR)$ to a given matrix $U$.
{In fact, every critical point is real and is obtained from a singular value decomposition of $U$ by setting $\min(m,n)-r$ singular values to zero, similarly as in \Cref{thm:EY}, see \cite[Example~2.3]{EDdeg}.}
Hence, the Euclidean distance degree of the determinantal variety $\cM_{r,m\times n}(\RR)$ is $\binom{\min(m,n)}{r}$.
By \eqref{eq:SEdegMrmn}, this number is also the squared-error degree, i.e., 
\begin{align}\label{eq:EDdegMrmn}
 \squerdeg (\cM_{r,m\times n}(\RR)) \,=\, \deged(\cM_{r,m\times n}(\RR)) \,=\, \binom{\min(m,n)}{r} \, . 
\end{align}

In our investigations, we will commonly perform base changes of the matrices in~$\cM_{m \times n}$.
For a subvariety $\cX \subset \cM_{m \times n}$ and any $T\in \GL_n$, we denote by $\cX^{\cdot T}$ the image of $\cX$ under the linear isomorphism
\begin{align}\label{eq:baseChangeLeft}
 \ \cdot T\colon \ \cM_{m\times n} \longrightarrow \cM_{m\times n}, \quad M \mapsto MT \, . 
\end{align}

\begin{lemma}\label{lem:baseChange}
    Let $\cX \subset \cM_{m \times n}$ be a subvariety and let $T\in \GL_n$.
    Then, $\dim(\cX^{\cdot T}) = \dim \cX$, $\deg(\cX^{\cdot T}) = \deg \cX$, 
    $\squerdeg(\cX^{\cdot T}) = \squerdeg(\cX)$,
    $\Sing(\cX^{\cdot T}) = \Sing(\cX)^{\cdot T}$, and 
    $(\cX^{\cdot T}) \cap \cM_{r, m \times n} = (\cX \cap \cM_{r, m \times n})^{\cdot T}$
    for any $r \le \min(m,n)$.
\end{lemma}
\begin{proof}
Since \eqref{eq:baseChangeLeft} is a linear isomorphism, it preserves the dimension and the degree, and maps regular points to regular points.
Due to the genericity of~$X$ in \Cref{def:squaredErrorDegree}, the squared-error degree is not affected by that isomorphism.
For the last assertion, we observe that every matrix $M \in \cM_{m \times n}$ satisfies that $\rank(MT) = \rank(M)$, since $T$ has full rank.
\end{proof}
We point out that the Euclidean distance degree (short, ED degree) is in general not preserved by the isomorphism~\eqref{eq:baseChangeLeft} as the following example demonstrates.

\begin{example}\label{ex:degedchange}
Let $m=1$ and $n=3$. Consider the circle $\cX \subset \cM_{1 \times 3}(\CC)$ defined by the equation $M_{1,1}^2+M_{1,2}^2-M_{1,3}^2 = 0$. The ED degree of $\cX(\RR)$ in the standard Euclidean distance
\begin{align}\label{eq:stanardEDdist} 
d(0,M) \,\coloneqq \, \lVert M \rVert ^2 \,=\, \mathrm{tr}(M M^\top) \,=\, M_{1,1}^2+M_{1,2}^2+M_{1,3}^2
\end{align}
is equal to~$2$. 
Throughout this article, we often consider permutation matrices and diagonalize them. 
For instance to diagonalize the permutation matrix $P\coloneqq\left( \begin{smallmatrix}
0&0&1 \\ 1 &0&0 \\ 0&1&0
\end{smallmatrix} \right)$, we use the symmetric Vandermonde matrix $$T \,\coloneqq \,V\!\left(1,\zeta_3,\zeta_3^2\right) \,=\, 
\begin{pmatrix}
    1&1&1 \\ 1 &\zeta_3&\zeta_3^2 \\ 1&\zeta_3^2&\zeta_3
\end{pmatrix},$$
where $\zeta_3$ is any primitive third root of unity.
In fact, we obtain $T^{-1}PT = \operatorname{diag}(1,\zeta_3^2,\zeta_3)$. The ED degree of $\cX^{\cdot T}(\RR)$ is \underline{not} equal to~$2$ anymore.  In fact, the ED degree of $\cX^{\cdot T}$ counts the critical points of the function $d(0,MT-U)$ over all $M \in \cX(\RR)$ for fixed, generic~$U$. Since~$U$ is generic, we can replace it by~$UT$.  Hence, the ED degree of~$\cX^{\cdot T}$ under the standard Euclidean distance \eqref{eq:stanardEDdist} is equal to the ED degree of $\cX$ under the modified Euclidean distance 
\begin{align*}
 d_T(0,M) \, \coloneqq \, d(0,MT)  \,=\, 3\cdot (M_{1,1}^2 + 2M_{1,2}M_{1,3}) \,.
 \end{align*}
A computation in \texttt{Macaulay2} shows that the ED degree of $\cX(\RR)$ under $d_T$ is~$4$.
\end{example}

However, the ED degree does keep unchanged under the multiplication by an orthogonal matrix by the right.
\begin{lemma}\label{lem:EDorth}
Consider a variety $\mathcal{X}\subset \cM_{m\times n}(\RR)$ and let $T\in O(n)$ be an orthogonal matrix. Then $\deged (\cX)=\deged(\cX^{\cdot T})$.
\end{lemma}
\begin{proof}
Let $U\in \cM_{m\times n}(\RR)$ be generic and $T$ an orthogonal $n\times n$ matrix. Then 
\begin{align} \label{eq:EDproblemOrthTransform}
\begin{split}
    \min_{X\in \cX} \,  \lVert XT-UT \rVert ^2 & \,=\, \min_{X\in \cX} \left( \operatorname{tr}\left((X-U)T\cdot ((X-U)T)^\top\right) \right)\\
    & \,=\, \min_{X\in \cX} \left( \operatorname{tr}\left((X-U)\cdot (X-U)^\top\right) \right)\\
    & \,=\, \min_{X\in \cX} \,  \lVert X-U \rVert ^2 ,
    \end{split}
\end{align}
resulting in the same number of critical points for minimizing over $\cX$ and $\cX^{\cdot T}$, resp.
\end{proof}

Depending on whether we study invariance or equivariance, we perform base changes on only one side of the matrices, or on both sides. For the latter, we write $M^{\sim T} := T^{-1}MT$ for given matrices $M \in \cM_{n \times n}$ and $T \in \GL_n$.
\begin{lemma}\label{lem:invAndEquivUnderbaseChange}
Consider the matrices $M \in \cM_{m \times m}$, $P \in \cM_{n \times n}$, and $T \in \GL_n$.
Then, $MP=M$ if and only if $M^{\cdot T} P^{\sim T} = M^{\cdot T}$.
In the case that $m=n$, we moreover have that $MP=PM$ if and only if $M^{\sim T}P^{\sim T} = P^{\sim T}M^{\sim T}$.
\hfill \qed
\end{lemma}
To proceed, we first recall the definition of a cycle. A permutation $\sigma \in \cS_n$ is 
a {\em cycle} if there exists a subset $\{a_1, a_2, \ldots, a_m\} \subseteq \{1, 2, \dots, n\}$ such that $\sigma(a_i)=a_{i+1}$ for $i=1,\ldots,m-1$, $\sigma(a_m)=a_1$, 
and $\sigma(x) = x$ for all $x \notin \{a_1, a_2, \dots, a_m\}$. Recall that every permutation can be uniquely decomposed into disjoint cycles (up to re-ordering the cycles).

In our investigations, we will make use of presentations of permutation matrices $P$ in different bases. The strategy is as follows. Step~$1$ consists in decomposing a permutation $\sigma\in \cS_n$ into disjoint cycles $\pi_1,\ldots,\pi_k$ of lengths $\ell_1,\ldots,\ell_k$; this brings the permutation matrix $P_\sigma$ of $\sigma$ into block  diagonal form, where each diagonal block is a circulant matrix. The second step is to diagonalize those  circulant matrices of sizes $\ell_1,\ldots,\ell_k$. Step~$3$ is optional and groups the columns corresponding to the same~eigenvalue. 

\pagebreak

\begin{procedure}\label{proc:basechange}
\begin{enumerate}[Step 1.]
\item[] 
\setcounter{enumi}{-1}
\item Represent $\sigma$ by the permutation matrix $P_\sigma\in \cM_{n \times n}(\{0,1\})$ with respect to the standard basis of~$\RR^n$, i.e., the $j$-th row of $P_\sigma$ is the transpose of the $\sigma(j)$-th standard unit vector of $\RR^n$.
\item Determine a permutation matrix $T_1\in \cM_{n\times n}(\{0,1\})$ such that  $P_\sigma^{\sim_{T_1}}=T_1^{-1}P_\sigma T_1$ is block diagonal whose blocks are circulant matrices $C_1,\ldots,C_k$ of the form 
\begin{align}\label{eq:circulant}
C_i \,=\,
{\small 
\begin{pmatrix}
0&&&&&1\\
1&0\\
&1 & & \ddots & &\\
&&&\ddots&\\
&&&&1&0
\end{pmatrix}} \in \cM_{\ell_i\times \ell_i}(\{0,1\}).
\end{align}
Each $C_i\in \cM_{\ell_i \times \ell_i}(\{0,1\})$ is a matrix representation of the cycle $\pi_i$ of length $\ell_i$, and has $\ell_i$-th roots of unity as eigenvalues, namely $\zeta_{\ell_i}^j$, where $j=0,\ldots,\ell_i-1$, and $\zeta_{n}$ denotes the primitive root of unity $e^{2\pi i/n}$. Depending on the lengths $\ell_i$ of the cycles, some of $C_i$'s might share common eigenvalues. Collect the eigenvalues of all the $C_i$'s in a set $\{ \lambda_1, \lambda_2, \ldots, \lambda_s \}$ together with their multiplicities~$b_1,b_2,\ldots,b_s$. Note that one of the $\lambda_i$'s is equal to~$1$ and for this~$\hat{i}$, $b_{\hat{i}}=k$.
\item Diagonalize each matrix $C_i$ from~\eqref{eq:circulant} via a matrix in~$\GL_n(\CC)$; this can be obtained via Vandermonde matrices $V(1,\zeta_{\ell_i},\ldots,\zeta_{\ell_i}^{\ell_i-1})$ as in \Cref{ex:degedchange}. The following block diagonal matrix then diagonalizes the block circulant diagonal matrix~$P_\sigma^{\sim_{T_1}}$ from~Step~1:
\begin{align}  
T_2 \,=\,  \begin{pmatrix}
V\!\left(1,\zeta_{\ell_1},\ldots,\zeta_{\ell_1}^{\ell_1-1}\right) & &   \\
&   \ddots & \\
& & && V\!\left(1,\zeta_{\ell_k},\ldots,\zeta_{\ell_k}^{\ell_k-1}\right)
\end{pmatrix} \, .
\end{align}
\item Group identical eigenvalues: determine $T_3\in \GL_n(\{0,1\})$ such that the matrix from Step~$2$ is block diagonal with blocks of the form~$\lambda_i \Id_{b_i}$, i.e., determine $T_3$ such that
\begin{align} \label{eq:Pfinal}
    T_3^{-1} \cdot {\left( P_\sigma^{\sim_{T_1}}\right)}^{\sim_{T_2}} \cdot T_3 \ =\ 
    \begin{pmatrix}
  \lambda_1\Id_{b_1} & & &   \\
  & \lambda_2 \Id_{b_2} &  &\\
  &  & \ddots & \\
  & & & & \lambda_s\Id_{b_s}  
\end{pmatrix} \, .
\end{align}
\end{enumerate}
\end{procedure}
\noindent 
We demonstrate Steps~$0$--$2$ of \Cref{proc:basechange} at an example.

\begin{example}\label{ex:proc}
Consider the permutation 
$\sigma= 
\begin{pmatrix}
1 & 2 & 3& 4 & 5\\
 3 & 5 & 4 & 1 & 2
\end{pmatrix}\in \cS_5$. 
Then
\begin{align*}
P_\sigma \,=\, 
\begin{pmatrix}
 0 & 0 & 1 & 0 & 0\\
 0 & 0 & 0 & 0 & 1\\
 0 & 0 & 0 & 1 & 0\\
 1 & 0 & 0 & 0 & 0 \\
 0 & 1 & 0 & 0 & 0
\end{pmatrix}
 \, \stackrel{\sim_{T_1}}{\mapsto} \,
\begin{pmatrix}
\begin{array}{c|c}
\begin{matrix}
0 & 0 & 1\\
1 & 0 & 0\\
0 & 1 & 0
\end{matrix} & \begin{matrix} 0 \end{matrix} \\ \hline
\begin{matrix} 0 \end{matrix} & 
\begin{matrix}
0 & 1\\ 1 & 0
\end{matrix}
\end{array}
\end{pmatrix}
\,\stackrel{\sim_{T_2}}{\mapsto}\,
\begin{pmatrix}
1 & &  & & \\
  & \zeta_3^2 & &  & \\
  & & \zeta_3  & & \\
  & & & 1 & \\
  & & & & -1\\
\end{pmatrix}
\end{align*}
with 
\begin{align*}
T_1\,=\,
\begin{pmatrix}
0 & 0 & 1 & 0 &0\\
0 & 0 & 0 & 1 & 0\\
0 & 1 & 0 & 0 & 0\\
1 & 0 & 0 & 0& 0\\
0 & 0 & 0 & 0 &1
\end{pmatrix} 
\in \cM_{5 \times 5}(\{0,1\}) \quad \text{and}\quad  T_2\,=\,
\begin{pmatrix}
\begin{array}{c|c}
\begin{matrix}
1 & 1 & 1\\
1& \zeta_3 & \zeta_3^2\\
1& \zeta_3^2 & \zeta_3 
\end{matrix} & \begin{matrix} 0 \end{matrix}\\ \hline
\begin{matrix}
  0
\end{matrix} & \begin{matrix}
1 & 1\\
1 & -1
\end{matrix}
\end{array} \end{pmatrix} 
\in \GL_5(\CC) \, , 
\end{align*}
where $\zeta_3$ denotes the primitive $3$rd root of unity~$e^{2\pi i/3}.$
\end{example}

\Cref{proc:basechange} requires complex base change matrices $T_2$.
We will present a real version of this procedure in \Cref{sec:real-irreduce}.
Both the complex and the real procedure turn a permutation matrix $P_\sigma$ into a block diagonal matrix~$P$ such that no two distinct blocks have eigenvalues in common. We can then easily determine the set of matrices that commute with~$P$, and we denote this vector space by~$C(P)$.

\begin{lemma}\label{lem:commute-block}
Let $P =[P_{i,j}] =\bigoplus_{i=1}^n P_{i,i}$ be a block diagonal matrix, where each block $P_{i,i}$ is square and such that sets of eigenvalues of the blocks are pairwise disjoint. Then the commutator $C(P)$ of $P$ consists precisely of those  block matrices $M =[M_{i,j}]= \bigoplus_{i=1}^n M_{i,i}$ for which each $M_{i,i}$ commutes with~$P_{i,i}$.
\end{lemma}
In the lemma, $M$ is assumed to be chopped into blocks of the same size pattern as~$P$.
\begin{proof}
Consider the matrix \(M = [M_{i,j}]\) as a block-matrix, where each block \(M_{i,j}\) has the same size as the corresponding block \(P_{i,j}\) in the block diagonal matrix \(P\). Our objective is to establish that if \(M\) commutes with \(P\), i.e., \(MP = PM\), then all off-diagonal blocks \(M_{i,j}\) are zero.
From the condition $MP=PM$, it follows that for all \(i\) and \(j\):
\begin{align}\label{eq:MijPjj}
M_{i,j} \cdot P_{j,j} \,=\, P_{i,i} \cdot  M_{i,j} \, . 
\end{align}
Hence $M_{i,j}P_{j,j}^2=P_{i,i}M_{i,j}P_{j,j}=P_{i,i}^2M_{i,j}$ by iterating~\eqref{eq:MijPjj}.
Consequently, for any univariate polynomial~\(p\), we have 
\begin{align}  \label{eq:polynomialMatrix}
M_{i,j} \cdot p(P_{j,j}) \,=\, p(P_{i,i}) \cdot M_{i,j}\, . 
\end{align}
Let \(p\) be the characteristic polynomial of \(P_{j,j}\). Thus, we have \(p(P_{j,j}) = 0\).
Writing $p=\prod_k (x-\lambda_k)$, where the $\lambda_k$ are the eigenvalues of $P_{j,j}$, we see that 
\(p(P_{i,i}) = \prod_k (P_{i,i}-\lambda_k I)\), for~$i \neq j$, is invertible, since
  \(P_{i,i}\) and \(P_{j,j}\) do not have any eigenvalue in common. Consequently, we deduce from \eqref{eq:polynomialMatrix} that $M_{i,j}$ is the zero matrix. We  conclude that \(M\) commutes with~\(P\) if and only if  $M_{i,i}$ commutes with $P_{i,i}$ for all~\(i=1,\ldots,n\).
\end{proof}

In the subsequent lemma, we determine the dimension of the eigenspaces of $P_{\sigma}$. We are going to leverage this information to compute the dimension of the commutator $C(P_{\sigma})$. 

\begin{lemma}\label{lem:dim-V_k}
Let $P_{\sigma}$ be a permutation matrix of size~$n$ with disjoint cycles, whose associated cyclic matrices $C_j$ are square of size~$\ell_j$. Then the  dimension of the eigenspace of~$P_\sigma$
corresponding to {an eigenvalue of the form} $e^{2 \pi i /l}$ is
\begin{align}
d_l \ \coloneqq\ \#\{ \, j \text{ such that } l|\ell_j \}\, .
\end{align}
In particular, $\dim_{\mathbb{C}}(C(P_{\sigma})) \,=\, \sum_l \varphi(l) \cdot d_l^2$. Here, $\varphi$ denotes Euler's $\varphi$-function, whose value $\varphi(l)$ is the order of the unit group of~$\,\ZZ /l \ZZ \, .$
\end{lemma}
\begin{proof}
We begin by observing that the eigenvalues of the matrix $C_j$ are the $\ell_j$-th roots of unity, of multiplicity one each. Consequently,  $e^{2 \pi i/l}$ is an eigenvalue of $C_j$ if and only if $l$ divides $\ell_j$. Hence, $d_l$ is the number of $j$ for which $l$ divides~$\ell_j$. Additionally, we find that also the eigenspace corresponding to the eigenvalue $e^{2 \pi i m/l}$, where $\gcd(m,l)=1$, has dimension~$d_l$.  Over the complex numbers, $P_\sigma$ is similar to the diagonal matrix $D_{\sigma} = [D_{l}^{(m)}]$, where $D_{l}^{(m)} = e^{2 \pi i m/l} \Id_{d_l}$ and $(l,m)$ runs over the set $\{ (l,m)\, | \, m\leq l,\ \gcd(m,l)=1\}$ of cardinality $\varphi(l)$. Employing  \Cref{lem:commute-block}, we  establish that the commutator $C(D_{\sigma})$ consists of arbitrary block diagonal matrices with claimed size of the blocks. Thus, we deduce from \Cref{lem:invAndEquivUnderbaseChange} that $\dim_{\mathbb{C}}(C(P_{\sigma})) = \sum_l \varphi(l) \cdot  d_l^2 \, .$
\end{proof}

\begin{remark}
In the notation of \Cref{proc:basechange}, $\{d_i \, | \,i\in \NN_{>0}, \,  \, d_i\neq 0\} =\{b_1,\ldots,b_s\}$ as sets. Moreover, $d_1=k$, i.e., the number of disjoint cycles into which $\sigma$ decomposes.
\end{remark}

\begin{lemma}
Let $\mathbb{F} \subset  \mathbb{K}$ be a field extension. Then for any $n\times n$ matrix $M \in \cM_{n\times n}(\mathbb{F})$, we have $\dim_{\mathbb{F}} (C_{\mathbb{F}}(M))=\dim_{\mathbb{K}} (C_{\mathbb{K}}(M))$.
\end{lemma}
\begin{proof}
Let $M$ be a matrix in $\cM_{n\times n}(\mathbb{F})$, and let $Q_{M,\mathbb{F}}: \cM_{n\times n}(\mathbb{F})\to \cM_{n\times n}(\mathbb{F})$ be a linear map that maps a matrix $N$ to its commutator with $M$, namely to $[M,N]=MN - NM$. Since the process of Gaussian elimination is independent of the choice of the field for the linear map $Q_{M,\cdot}$, it follows that $\dim_{\mathbb{F}} (\ker(Q_{M,\mathbb{F}}))=\dim_{\mathbb{K}} (\ker(Q_{M,\mathbb{K}}))$.
\end{proof}

\begin{corollary}\label{cor:field-exten}
\Cref{lem:dim-V_k} holds true also when the base field~$\mathbb{F}$ is $\mathbb{R}$ or $\mathbb{Q}$. \hfill \qed
\end{corollary}

\begin{example} \label{ex:compute_dl}
The blocks of the matrix in \eqref{example:permute-99} represent the three disjoint cycles $(1 \, 2 \, 3 \, 4)$, $(5 \, 6 \, 7 \,8 ),$ and $(9)$ of length $4$, $4$, and $1$, respectively. Hence the only non-zero $d_i$'s are\linebreak $d_1 =k= 3$, $d_2 = 2$, and $d_4 = 2$. Therefore, by \Cref{lem:dim-V_k} and \Cref{cor:field-exten}, we have $\dim(C(P_\sigma)) = \varphi(1) \cdot 3^2 + \varphi(2) \cdot 2^2 +  \varphi(4) \cdot 2^2= 9+4+8=21$, in coherence what was obtained in \Cref{sec:warmup}.
After the full base change in \Cref{proc:basechange} such that $P_\sigma$ becomes the diagonal matrix in \Cref{eq:Pfinal}, 
the matrices that commute with that diagonal matrix are those of the form
\begin{align}
	\begin{pmatrix}
 \begin{array}{c|c|c|c}
 \begin{matrix}
		\alpha_{11} & \alpha_{12} & \alpha_{13} \\
		\alpha_{21} & \alpha_{22} & \alpha_{23}  \\
		\alpha_{31} & \alpha_{32} & \alpha_{33}
  \end{matrix} & \begin{matrix} 0 \end{matrix}& \begin{matrix} 0 \end{matrix} & \begin{matrix} 0 \end{matrix} \\ \hline 
  \begin{matrix} 0 \end{matrix} & 
  \begin{matrix}
 \beta_{11} & \beta_{12}\\
\beta_{21} & \beta_{22}
 \end{matrix} 
 & \begin{matrix} 0 \end{matrix} & \begin{matrix} 0 \end{matrix}\\ \hline 
 \begin{matrix} 0 \end{matrix} & \begin{matrix} 0 \end{matrix} &
 \begin{matrix}
	\gamma_{11} & \gamma_{12}\\
	 \gamma_{21} & \gamma_{22} 
  \end{matrix} & \begin{matrix} 0 \end{matrix} \\ \hline 
 \begin{matrix} 0 \end{matrix} & \begin{matrix} 0 \end{matrix} & \begin{matrix} 0 \end{matrix} &\begin{matrix}
 \delta_{11} & \delta_{12} \\
\delta_{21} & \delta_{22} 
\end{matrix}
  \end{array}
	\end{pmatrix},
\end{align}
with scalars $\alpha_{ij}$, $\beta_{ij}$, $\gamma_{ij}$, $\delta_{ij}$. One sees that there are $17$ ways how such a matrix can be of rank~$3$: Either the first block has rank $3$ and all others are zero, or one block has rank~$2$ while another has rank~$1$ (there are $12$ possible choices of such blocks), or one block is zero while all the others have rank~$1$ (which gives $4$ possible choices).
These are precisely the $17$ irreducible components of $\cE^\sigma_{3,9\times 9}(\CC)$ that were predicted in \Cref{sec:warmup}.
\end{example}

\section{Invariance under permutation groups}\label{sec:inv}
We here study linear maps $f \colon \, \RR^n\to \RR^m$ and are going to deal with determinantal subvarieties of~$\cM_{m\times n}$. For invariance, the action of the considered group is required on the input space~$\RR^n$ only. 
Let $G\leq \cS_n$ be an arbitrary subgroup of the symmetric group. The linear map $ f$ is {\em invariant under $G$} if
\begin{align}
f\circ \sigma \, = \, f 
\end{align}
for all permutations $\sigma \in G$.

\subsection{Reduction to cyclic groups}\label{sec:invarbitr}
We denote the  columns of a matrix $M \in \cM_{m \times n}$  by $M_1, M_2, \ldots, M_n$. 
For a decomposition $\sigma=\pi_1\circ \cdots \circ  \pi_k \in \cS_n$ into disjoint cycles, we denote by $\cP(\sigma)=\{A_1,\ldots,A_k\}$ its induced partition of the set~$[n]$. The $A_i\subset [n]$ fulfill $\cup_{i=1}^kA_i=[n]$ and $A_i\cap A_j=\emptyset$ whenever $i\neq j$.

\begin{example}
Let $\sigma=(1 \, 3 \, 4)(2 \, 5)\in \cS_5.$ Its induced partition of $[5]=\{1,2,3,4,5\}$ is $\cP(\sigma)=\{ \{1,3,4\},\{2,5\}\}. $ Note that different permutations might induce the same partition: the permutation $\eta=(1\, 4 \, 3)(2 \, 5)\neq \sigma$ gives rise to the partition $\cP(\eta)=\cP(\sigma)$ of the set~$[5]$.
\end{example}

\begin{lemma}\label{prop:diminvcyclic}
Let $\sigma \in \cS_n$ and consider its   decomposition $\sigma=\pi_1\circ \cdots \circ \pi_k$ into $k$ disjoint cycles $\pi_1,\ldots,\pi_k$. 
Then
\begin{align}
    \cI^\sigma_{m \times n} \,=\, \left\{ M \in \cM_{m \times n} \mid \forall A \in \mathcal{P}(\sigma) \ \forall i,j \in A: \, M_i = M_j \right\} \, \cong \, \cM_{m \times k} \, .
\end{align}
\end{lemma}
\begin{proof}
Consider the partition $\cP(\sigma)=\{A_1,\ldots,A_k\}$ of $[n]$ induced by the decomposition $\sigma=\pi_1\circ \cdots \circ \pi_k$. Assuming invariance of $M$ under $\sigma$, each cycle $\pi_i$ of $\sigma$ forces some of the columns of $M$ to coincide, namely the columns indexed by~$A_i\subset [n]$.
For each $A_i\in \cP(\sigma)$, we need to remember only one column $M_{A_i}$ of~$M$. For each~$i$, $\operatorname{length}(\pi_i)$ many identical copies of~$M_{A_i}$ are listed as columns in~$M$. Deleting all columns except the $M_{A_i}$ gives a linear isomorphism $\cI^\sigma_{m \times n} \cong \cM_{m \times k}$.
\end{proof}

Thus, invariance under a permutation $\sigma$ depends only on the partition~$\cP(\sigma)$. We also see that any invariant matrix can have at most $k$ pairwise distinct columns:
\begin{corollary}\label{cor:rankinv}
 If a linear function $f\colon \RR^n \to \RR^m$ is invariant under $\sigma\in \cS_n$, then its rank is at most~$k$, the number of disjoint cycles into which~$\sigma$ decomposes. \hfill \qed
\end{corollary}

Invariance under a permutation group~$G$ with arbitrarily many generators can be reduced to invariance under cyclic groups, which we make precise in the following proposition. 
\begin{lemma}\label{prop:reducetocyclic}
Let $G = \langle \sigma_1, \sigma_2, \ldots, \sigma_g \rangle \leq \cS_n$ be a permutation group. There exists \mbox{$\sigma \in \cS_n$} such that $\cI_{m\times n}^G=\cI_{m\times n}^\sigma$. 
In fact, this equality holds
for any $\sigma \in \cS_n$ whose induced partition~$\cP(\sigma)$ is the finest common coarsening of the partitions $ \cP(\sigma_1),\cP(\sigma_2),\ldots,\cP(\sigma_g) $.
\end{lemma}
\begin{proof}
Decompose each $\sigma_i$ into pairwise disjoint cycles $\pi_{i,1}\circ \cdots \circ \pi_{i,k_i}.$
Invariance of a matrix $M$ under $\sigma_i$ forces some of the columns of $M$ to coincide and depends only on the partition $\cP(\sigma_i)$ of $[n]$. Any additional $\sigma_j$, $j\neq i$, forces further columns of $M$ to coincide. Invariance of~$M$ under $G$ is hence described by the finest common coarsening of $\cP(\sigma_1),\ldots,\cP(\sigma_g)$.
\end{proof}

\subsection{Characterization of 
$\cI_{r,m\times n}^G$}\label{sec:diminv}
By \Cref{prop:reducetocyclic}, it is sufficient to consider cyclic groups. 
By omitting repeated columns as in \Cref{prop:diminvcyclic}, we see that $\cI_{r, m \times n}^G$ is linearly isomorphic to $\cM_{\min(r,k), m \times k}$.
\begin{proposition}\label{prop:isoinvM}
Let $G=\langle \sigma \rangle \leq \cS_n$ be cyclic, and $\sigma=\pi_1\circ \cdots \circ \pi_k$ a decomposition into pairwise disjoint cycles~$\pi_i$. The variety  $\cI_{r,m\times n}^G$ is isomorphic to the determinantal variety~$\cM_{\min(r,k),m\times k}$ via a linear morphism that deletes repeated columns:
\begin{align}\label{eq:psi}
\psi_{\cP(\sigma)} \colon \ \cI_{r,m\times n}^G \to \cM_{\min(r,k),m\times k} \, .
\end{align}
\end{proposition} 
\begin{proof} For each $A_i\in \cP(\sigma),$ $\Psi_{\cP(\sigma)}$ remembers the column $M_{A_i}$ of~$M$. Intersecting with $\cM_{r,m\times n}$ imposes linear dependencies on the columns of~$M$. This mapping rule is linear in the entries of~$M$. To invert~$\psi_{\cP(\sigma)}$, one  needs to remember $\cP(\sigma)$ as datum.
\end{proof}
\begin{example}[$m=2$, $n=5$, $r=1$] 
Let $\sigma=(1 \, 3 \, 4)(2 \, 5)\in \cS_5$ and hence $k=2$. Any invariant matrix $M\in \cM_{2\times 5}(\RR)$ is of the form $\left( \begin{smallmatrix}
        \alpha & \gamma & \alpha & \alpha & \gamma \\
        \beta & \delta & \beta & \beta & \delta
    \end{smallmatrix}\right)$ for some $\alpha,\beta,\gamma,\delta\in \RR$. The rank constraint $r=1$ imposes that $(\gamma,\delta)=\lambda\cdot(\alpha,\beta)^\top$ for some $\lambda \in \RR,$ where we assume w.l.o.g. that $(\alpha,\beta)\neq (0,0)$ in the case of rank one. The morphism~\eqref{eq:psi} hence is
    \begin{align}
        \psi_{\cP(\sigma)}\colon \ \begin{pmatrix}
             \alpha & \lambda \alpha & \alpha & \alpha & \lambda\alpha \\
        \beta & \lambda \beta & \beta & \beta & \lambda \beta
        \end{pmatrix} \, \mapsto \, \begin{pmatrix}
            \alpha & \lambda \alpha \\ \beta & \lambda \beta
        \end{pmatrix} .
    \end{align}
To invert the morphism $\psi_{\cP(\sigma)}$, one reads from $\cP(\sigma)$ how to copy and paste the columns $(\alpha,\beta)^\top$ and $(\lambda \alpha,\lambda \beta)^\top$ to recover the $2\times 5$ matrix one started with.
\end{example}

\begin{corollary}\label{cor:dimdeginv}
In the setup of \Cref{prop:isoinvM}, one has
\begin{align}\begin{split}
\dim \left(  \cI_{r,m\times n}^\sigma \right) &\,=\ \min(r,k)\cdot(m+k-\min(r,k)) \, , \\ \deg \left( \cI_{r,m\times n}^\sigma \right) &\,=\ \prod_{i=0}^{k-\min(r,k)-1} \frac{(m+i)!\cdot i!}{(\min(r,k)+i)!\cdot (m-(\min(r,k)+i)!} \, , \\ \Sing(\cI_{r,m\times n}^\sigma) &\,=\ \psi^{-1}_{\cP(\sigma)}\left( \cM_{r-1,m\times k} \right) 
\text{ if } r < \min(m,k), 
\text{ and empty otherwise}.
\end{split}\end{align}
\end{corollary}
\begin{proof}
The statements are an immediate consequence of \Cref{prop:isoinvM} combined with Equations \eqref{eq:dimMrmn} and~\eqref{eq:degMrmn}, and \Cref{lem:sing-mat}. 
\end{proof}

We now consider the problem of finding an invariant matrix {of low rank} that is closest (with respect to the Euclidean distance) to a given matrix---or, more generally, squared-error loss minimization
\begin{align}\label{eq:minimizeInv}
    \argmin_{M \in \, \cI^\sigma_{r, m \times n}} \Vert MX - Y \Vert^2_F \, .
\end{align}
We will see in \Cref{sec:networkDesign} that there are linear neural networks whose function space is~$\cI^\sigma_{r, m \times n}$. 
Training such a network with the squared-error loss aims at solving~\eqref{eq:minimizeInv}.

\begin{proposition} \label{prop:EYinv}
For all matrices $X \in \RR^{n \times d}$ and $Y \in \RR^{m \times d}$ with $\rank(XX^\top) = n$,  squared-error loss minimization \eqref{eq:minimizeInv} can be solved using Eckart--Young on $\cM_{\min(r,k), m \times k}(\RR) $, where~$k$ is the number of disjoint cycles in $\sigma$.
In particular, 
\begin{align}\label{eq:SEdegreeInv}
    \squerdeg (\cI^\sigma_{r, m \times n}) \,=\, \deged (\cI^\sigma_{r, m \times n}) \,=\, \deged (\cM_{\min(r,k), m \times k}) = \binom{\min(m,k)}{\min(r,k)}.
\end{align}
\end{proposition}
\begin{proof}
We make use of the linear isomorphism $\psi_{\mathcal{P}(\sigma)}$ from \eqref{eq:psi} that {deletes} columns of a given matrix~$M$.
For each set $A \in \mathcal{P}(\sigma)$, it keeps only one of the columns indexed by $A$.
From the data matrix $X \in \RR^{n \times d}$, we define a new matrix $\tilde X \in \RR^{k \times d}$ by summing all the rows of $X$ that are indexed by the same set $A \in \mathcal{P}(\sigma)$.
That way, we have $M \cdot X = \psi_{\mathcal{P}(\sigma)}(M) \cdot \tilde X$.
Hence, minimizing the squared-error loss on $\cI^\sigma_{r, m \times n}$ as in \eqref{eq:minimizeInv} is equivalent to {solving}
\begin{align} \label{eq:minimizeInvReformulation}
    \argmin_{\tilde{M} \in \, \cM_{\min(r,k), m \times k}} \Vert \tilde M \tilde X - Y \Vert^2_F \, .
\end{align}
Note that $\tilde{X}\tilde{X}^\top$ is of full rank $k$ since $XX^\top$ is a full-rank matrix. 
Therefore, we can apply \Cref{lem:squaredErrorED} and solve \eqref{eq:minimizeInvReformulation}
via Eckart--Young as in
\Cref{ex:FisEverything}. 
By choosing $X \in \RR^{n \times n}$ to be either generic or the identity, we observe that \eqref{eq:SEdegreeInv} holds for the squared-error degree and the Euclidean distance degree of $\cI^\sigma_{r, m \times n}$, respectively. 
\end{proof}

\subsection{Parameterizing invariance and network design} \label{sec:networkDesign}
In this section, we investigate which implications imposing invariance has on the individual layers of a linear autoencoder. 
Recall that a  matrix $M \in \cM_{m\times n}(\KK)$ has rank $r\le \min(m,n)$ if and only if there exist rank-$r$ matrices $A \in \cM_{m\times r}(\KK)$ and $B \in \cM_{r\times n}(\KK)$ such that~$M = AB$.
The factors of an invariant matrix do not need to be invariant, as the following example demonstrates. To be more precise, this question makes sense to be asked a priori only for the first layer, $B$, since the group~$G$ acts on $\RR^n$ only.

\begin{example}
Let $\sigma=(1 \, 2)\in \cS_3$ and consider 
\begin{align}
 A \,=\, 
\begin{pmatrix}
    1 & 1 \\
    1& 1\\
    0 & 0
\end{pmatrix} 
\quad \text{and} \quad  B \,=\, 
\begin{pmatrix}
     1 & 2 & 1 \\
     2 & 1 & 1
\end{pmatrix}  .
\end{align}
Observe that $B$ is not invariant under $\sigma$, but the product $AB$ is.
\end{example}

\begin{lemma}[{\cite[Proposition~22]{purespurcrit}}]\label{prop:fiberAB}
Let $r\leq \min(m,n)$. Denote by \begin{align}
\mu\colon\, \cM_{m\times r}\times \cM_{r\times n}, \quad  (A,B) \mapsto A\cdot B,
\end{align}
the multiplication map. 
If $\rank(M)=r$ and $M=\mu(A,B)$, then its fiber is
\begin{align}
    \mu^{-1}(M) \,=\, \left\{  \left( AT^{-1},TB \right) \, | \, T\in\GL_r(\KK) \right\} \, \subset \,  \cM_{m\times r}\times \cM_{r\times n}\, .
\end{align}
\end{lemma}

For $\sigma=\pi_1\circ \cdots \circ \pi_k \in \cS_n$ with pairwise disjoint cycles $\pi_i$, we denote by $E_{\cP(\sigma)}$ the $k\times n$ matrix that has the $i$-th standard basis vector $e_i\in \RR^k$ in column $j$ for all $j\in A_i$, $i=1,\ldots,k$. In particular, $e_i$ occurs $\text{length}(\pi_i)$ many times as a column of $E_{\cP(\sigma)}$, and $\operatorname{rank}(E_{\cP(\sigma)})=k$.

\begin{lemma}\label{prop:invMfators}
Right-multiplication with $E_{\cP(\sigma)}$ is precisely the inverse of the linear isomorphism $\psi_{\cP(\sigma)}$ from \eqref{eq:psi}. In other words,
any matrix $M \in \cI^\sigma_{m \times n}$ 
factorizes as
\begin{align}
    M \ =\  \psi_{\cP(\sigma)}(M) \cdot E_{\cP(\sigma)} \, .
\end{align}
\end{lemma}
\begin{proof} 
Let $M$ be invariant under $\sigma\in \cS_n$.
For each $A_i\in \cP(\sigma)$, the map $\psi_{\cP(\sigma)}$ remembers one of the columns indexed by $A_i$, and collects them as the columns of the $m\times k$ matrix $\psi_{\cP(\sigma)}(M)$.
The remaining columns of $M$ are copies of the columns of $\psi_{\cP(\sigma)}(M)$. The repetition pattern is encoded in~$E_{\cP(\sigma)}$.
\end{proof}

\begin{example}[$n=4$]
Let $\sigma=(1 \, 3)(2)(4)\in \cS_4$, so that $k=3$. The matrix 
\begin{align}
M \,=\, \begin{pmatrix}
    1 & 0 & 1 & 1\\
    2 & 4 & 2 & 1\\
    3 & 4 & 3 & 1\\
    0 & 2 & 0 & 1
\end{pmatrix} \,=\, \begin{pmatrix}
    1 & 0 & 1 \\
    2 & 4 & 1\\
    3 & 4 & 1\\
    0 & 2 & 1
\end{pmatrix}\cdot \underbrace{\begin{pmatrix}
    1 & 0 & 1 & 0\\
    0 & 1 & 0 & 0\\
    0 & 0 & 0 & 1
\end{pmatrix}}_{=\, E_{\{\{1,3\},\{2\},\{4\}\}}}
\end{align}
is invariant under $\sigma$ and factorizes as claimed in \Cref{prop:invMfators}.
\end{example}

\begin{proposition}\label{cor:autoencoderFact}
Let $\sigma \in \cS_n$ be decomposable into $k$ pairwise disjoint cycles with~$k \leq m$.
Then the two-layer factorizations of the matrices in $\cI^\sigma_{m \times n}$ that have the maximal possible rank $k$ are precisely those factorizations whose first layer is invariant under $\sigma$. In symbols, 
\begin{align}\begin{split} \label{eq:arbitratyFactorizationInv}
    \left\{ (A,B) \in \cM_{m \times k} \times \cM_{k \times n} \mid AB \in \cI^\sigma_{m \times n}, \,\rank(AB)=k\right\} \qquad \qquad \\
 = \, \left\{ A \in \cM_{m \times k} \mid \rank(A) = k \} \times \{ T \cdot E_{\cP(\sigma)} \mid  T\in \GL_k \right\} .
\end{split}
\end{align}
\end{proposition}
\begin{proof}
For $A \in \cM_{m \times k}$ of full rank $k$ and $T \in \GL_k$, the product $AT E_{\cP(\sigma)} $ has rank $k$ and is invariant under $\sigma$ by \Cref{prop:invMfators}.
For the converse direction, consider $(A,B)$ as in the first row of \eqref{eq:arbitratyFactorizationInv}.
Since the product $AB$ has rank $k$, both factors $A$ and $B$ must be of full rank.
By \Cref{prop:invMfators}, we have that
$AB = \psi_{\cP(\sigma)}(AB) \cdot E_{\cP(\sigma)} $.
Moreover, \Cref{prop:fiberAB} implies that
$(A,B) = (\psi_{\cP(\sigma)}(AB) T^{-1}, T E_{\cP(\sigma)}) $ for some $T \in \GL_k$.
\end{proof}

This 
tells us that linear autoencoders are well-suited for expressing invariance when one imposes appropriate weight-sharing on the encoder.
More precisely, the decoder factor $A$ in \eqref{eq:arbitratyFactorizationInv} is an arbitrary full-rank matrix, but the encoder factor $T\cdot E_{\cP(\sigma)}$ 
has repeated columns.
We impose this repetition pattern via weight-sharing on the encoder.
\Cref{cor:autoencoderFact} states that invariant matrices naturally lie in the function space of such an autoencoder.

\begin{definition}
Given any
permutation $\sigma \in \cS_n$, we say that an encoder $\RR^n \to \RR^r$ has \mbox{\em$\sigma$-weight-sharing} if its representing matrices satisfy the following: for every set $S \in \cP(\sigma)$, the columns indexed by the elements in $S$ coincide, and no additional weight-sharing is~imposed. 
\end{definition}

\begin{example}\label{ex:weightpic}
We revisit \Cref{ex:proc}. The invariance of a matrix $M=A\cdot B\in \cI_{2,5\times 5}^\sigma$ forces the encoder factor $B$ to fulfill the  weight-sharing property depicted in \Cref{fig:weightshar}. Which weights have to coincide is to be read from the color labeling in the figure.
\begin{figure}[h]
\begin{tikzpicture}[scale=0.24]
	\node[shape=circle,draw=black] (A) at (-5,6) {};
	\node[shape=circle,draw=black] (B) at (-5,3) {};
	\node[shape=circle,draw=black] (C) at (-5,0) {};
	\node[shape=circle,draw=black] (D) at (-5,-3) {};
    \node[shape=circle,draw=black] (E) at (-5,-6) {};
    \node[shape=circle,draw=black] (F) at (15,2) {};
	\node[shape=circle,draw=black] (G) at (15,-2) {};
	\path[very thick, Blue] (A) edge node[left] {}  (F); 
    \path[very thick, Aquamarine]  (A) edge node[left] {} (G);
    \path[very thick, OrangeRed] (B) edge node[left] {}  (F); 
    \path[very thick, BurntOrange]  (B) edge node[left] {} (G);
    \path[very thick, Blue] (C) edge node[left] {}  (F); 
    \path[very thick, Aquamarine]  (C) edge node[left] {} (G);
    \path[very thick, Blue] (D) edge node[left] {}  (F); 
    \path[very thick, Aquamarine]  (D) edge node[left] {} (G);
    \path[very thick, OrangeRed] (E) edge node[left] {}  (F); 
    \path[very thick, BurntOrange]  (E) edge node[left] {} (G);
	\node[] at (5,-6) {$B$};     
\end{tikzpicture}
\caption{The $\sigma$-weight-sharing property imposed on the encoder by $\sigma=(1 \, 3 \,4)(2 \, 5)$.}
\label{fig:weightshar}
\end{figure}
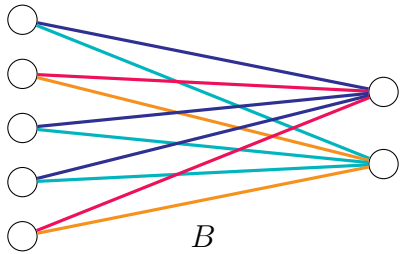
\end{example}

\begin{proposition}\label{prop:fctinvweight}
Let $\sigma \in \cS_n$ be a permutation consisting of $k$ disjoint cycles with $k\leq m$, and let $r \leq k$.
Consider the linear network $\RR^n \to \RR^r \to \RR^m$ with two fully-connected layers and $\sigma$-weight-sharing imposed on the first layer $\RR^n \to \RR^r$. Its function space is~$\cI^\sigma_{r, m \times n}(\RR)$.
\end{proposition}
\begin{proof}
	Every matrix in the function space of the network is of the form $AB$ such that the first-layer factor $B \in \cM_{r \times n}$ has repeated columns according to $\cP(\sigma).$
Therefore, the product $AB$ has the same repetition in its columns, i.e., $AB \in \cI^\sigma_{m \times n}$. Since $\rank(AB) \leq \rank(B) \leq r$, we conclude that $AB \in \cI^\sigma_{r, m \times n}(\RR)$.
For the converse direction, consider $M \in \cI^\sigma_{r, m \times n}(\RR)$. By \Cref{prop:invMfators}, that matrix can be factorized as $M  =  \psi_{\cP(\sigma)}(M) \cdot E_{\cP(\sigma)}$. If $r=k$, that factorization is compatible with the network and we are done. 
Thus, it is left to consider the case $r<k$. Note that $E_{\cP(\sigma)}$ has full rank $k$ by construction since the standard basis vectors in the columns of that matrix correspond to the $k$ cycles of $\sigma$.
Hence, we have that $\rank(\psi_{\cP(\sigma)}(M)) = \rank(M) \leq r$. Therefore, we can factorize $\psi_{\cP(\sigma)}(M) = A' B'$, where $A' \in \cM_{m \times r}$ and $B' \in \cM_{r \times k}\, .$
The factorization $M = A' \cdot \left( B' \cdot E_{\cP(\sigma)}\right)$ is compatible with the network. 
\end{proof}

\subsection{Induced filtration of \texorpdfstring{$\cM_{r,m\times n}$}{M}}
If an $m\times n$ matrix $M$ is invariant under $\sigma\in \cS_n$, then it is also invariant under every permutation $\eta\in \cS_n$ whose associated partition $\cP(\eta)$ of $[n]$ is a refinement $\cP(\eta)\prec \cP(\sigma)$ of $\cP(\sigma)$. This induces a filtration $\cI_{m\times n}^\bullet$ of the variety $\cM_{m\times n}$, which is indexed by partitions of~$[n]$, and by intersecting with $\cM_{r,m\times n}$, we obtain a filtration $\cI_{r,m\times n}^\bullet$ of the variety $\cM_{r,m\times n}$. The set $\cP([n])$ of partitions of $[n]$ equals $\cS_n/\sim$, where we identify $\sigma_1\sim \sigma_2$ if and only if $\cP(\sigma_1)=\cP(\sigma_2)$. Here, $\cI_{r,m\times n}^\cP$ denotes any $\cI_{r,m\times n}^\sigma$ for which $\cP(\sigma)=\cP$. As we saw earlier on, the variety $\cI_{r,m\times n}^\sigma$ depends only on $\cP(\sigma)$, but not on $\sigma$ itself, hence this notion is well-defined. 
Together with refinements of partitions, the set of partitions of $[n]$ is a partially ordered set. Define the category $\underline{\Part}_{[n]}^\prec$ whose set of objects is $\cP([n])$, and a morphism from~$\cP_1$ to~$\cP_2$ whenever $\cP_1\prec \cP_2$. By $\underline{\Subv}_{\cM_{r,m\times n}}^\subset$, we denote the category whose objects are subvarieties of $\cM_{r,m\times n}$, and the inclusion as morphism between $U_1,U_2\in \underline{\Subv}_{\cM_{r,m\times n}}^\subset $ whenever $U_1\subset U_2$. This formulation gives rise to the functor
\begin{align}
\underline{\Part}_{[n]}^\prec\longrightarrow \underline{\Subv}_{\cM_{r,m\times n}}^\subset \, , \quad \cP \,\mapsto\, \cI_{r,m\times n}^\cP \, .
\end{align}

\begin{remark}
The opposite category $\underline{\Part}_{[n]}^{\prec,\text{op}}$ of $\underline{\Part}_{[n]}^\prec$ is $\underline{\Part}_{[n]}^\succ$, i.e., partitions of $[n]$ with coarsenings $\succ$ of partitions as morphisms. In this formulation, the finest common coarsening of partitions $\cP_1,\ldots,\cP_k$ then is their inverse limit.
\end{remark}

\section[Equivariance under cyclic subgroups of the symmetric group]{Equivariance under cyclic subgroups of~$\cS_n$}\label{sec:equiv}
In this section, we address equivariance under cyclic subgroups $G=\langle \sigma \rangle\leq \cS_n$. For the cyclic group $G=\langle \sigma \rangle$, a linear map $f:\mathbb{R}^n \to \mathbb{R}^n$ is equivariant under $G$ if
\begin{align}
f\circ \sigma \, = \, \sigma \circ f. 
\end{align}
We explore the irreducible components {of the algebraic variety} $\cE_{r,n\times n}^\sigma\subset \cM_{r,n\times n}$, both over $\RR$ and over~$\CC$, and {quantities} such as their dimensions, degrees, and singular loci.
We parameterize the real components via autoencoders and discuss minimizing both the squared-error loss and the standard Euclidean distance on such a component.

\subsection{Characterizing equivariance}
We consider a permutation $\sigma = \pi_1 \circ \cdots \circ \pi_k \in \cS_n$ with a decomposition into pairwise disjoint cycles of lengths $\ell_1, \ldots, \ell_k$. 
We reorder the entries in $\RR^n$ as in Step~1 of \Cref{proc:basechange} such that the permutation matrix $P_\sigma$ becomes block diagonal with cyclic blocks as in \eqref{eq:circulant}.
Those square blocks have sizes $\ell_1, \ldots, \ell_k$.

We are interested in matrices $M \in \cM_{n \times n}$ that are equivariant under $\sigma$. For that, we divide $M \in \cM_{n \times n}$ into blocks following the same pattern as in $P_\sigma$:
$M$ has square diagonal blocks $M^{(ii)}$ of size $\ell_i \times \ell_i$ and rectangular off-diagonal blocks $M^{(ij)}$ of size $\ell_i \times \ell_j$.
We now show that the equivariance of $M$ under $\sigma$ means that its blocks have to be {\em (rectangular) circulant matrices}. 
We can observe this property in our previous example~\eqref{eq:matrotequi}.
We call a (possibly non-square) matrix {\em circulant} if each row is a copy of the previous row, cyclically shifted one step to the right, and also each column is a copy of the previous column, cyclically shifted one step downwards.
Some examples of such matrices are shown in \Cref{eq:specialcirc}. 
A circulant matrix of size $\ell_i \times \ell_j$ has at most $\gcd(\ell_i,\ell_j)$ different entries.
\begin{align}\begin{split}\label{eq:specialcirc}
\begin{pmatrix}
    \alpha^{(ij)} & \beta^{(ij)} & \gamma^{(ij)} \\
    \gamma^{(ij)} & \alpha^{(ij)} & \beta^{(ij)} \\
    \beta^{(ij)} & \gamma^{(ij)} & \alpha^{(ij)}
\end{pmatrix}, \quad \begin{pmatrix}
    \alpha^{(ij)} & \alpha^{(ij)}&\alpha^{(ij)}\\
    \alpha^{(ij)} & \alpha^{(ij)} & \alpha^{(ij)}\\
    \alpha^{(ij)} & \alpha^{(ij)} & \alpha^{(ij)}\\
    \alpha^{(ij)} & \alpha^{(ij)} & \alpha^{(ij)}
\end{pmatrix}, \\ \begin{pmatrix}
    \alpha^{(ij)} & \alpha^{(ij)}& \alpha^{(ij)}\\
    \alpha^{(ij)} & \alpha^{(ij)} & \alpha^{(ij)}
\end{pmatrix}, \quad \begin{pmatrix}
    \alpha^{(ij)} & \beta^{(ij)} & \alpha^{(ij)} & \beta^{(ij)}\\
    \beta^{(ij)} & \alpha^{(ij)} & \beta^{(ij)} & \alpha^{(ij)}
\end{pmatrix} .
\end{split}
\end{align}

\begin{proposition}\label{prop:blockform}
The  matrix $M \in \cM_{n \times n}$ is equivariant under~$\sigma$
if and only if each block $M^{(ij)}$ of $M$ is a (possibly non-square) circulant matrix. 
\end{proposition}
\begin{proof}
Since $P_\sigma$ is a block diagonal matrix with cyclic blocks $C_1, \ldots, C_k$, the equivariance condition $MP=PM$ means that $M^{(ij)}C_j=C_iM^{(ij)}$ for all $i,j$. The multiplication by $C_i$ from the left cyclically permutes the rows of $M$, and the multiplication by $C_j$ from the right permutes the columns of~$M$. Since the resulting matrices need to coincide, it follows that the block $M^{(ij)}$ has to be circulant.
\end{proof}

\subsubsection{Irreducible components of $\cE_{r,n\times n}^\sigma(\CC)$}\label{sec:dimequi} 
Our focus now shifts to exploring the intriguing algebraic properties arising from the intersection of the determinantal variety $\cM_{r,n\times n}$ with the linear space $C(P_{\sigma}) = \cE^\sigma_{n \times n}$.
Their intersection $\cE_{r,n\times n}^\sigma = \cM_{r,n\times n}\cap C(P_\sigma)$ is an algebraic set that is, in general, reducible.
In the following statement, we use the notation from \Cref{lem:dim-V_k}.

\begin{theorem}\label{thm:irreducible-dim}
There is a one-to-one correspondence between the irreducible components of $\cE_{r,n\times n}^\sigma(\CC)$ and the {non-negative} integer solutions $\mathbf{r}=(r_{l,m})$ of
\begin{equation}\label{eq:number-sol}
\sum_{l \geq 1} \sum_{m \,\in\, (\ZZ /l \ZZ)^\times} r_{l,m} \ =\ r \, ,\quad  \text{ where } \, 0 \le r_{l,m} \le d_l\, .
\end{equation}  
The irreducible component $\cE_{r,n\times n}^{\sigma,\mathbf{r}}(\CC)$ corresponding to such an integer solution $\mathbf{r}$ after the base change from \Cref{proc:basechange} is the following direct product of determinantal varieties:
\begin{align} \label{eq:complexComponent}
    \prod_{l \geq 1} \prod_{m \,\in\, (\ZZ /l \ZZ)^\times} \cM_{r_{l,m}, d_l \times d_l}(\CC).
\end{align}
\end{theorem}
\begin{proof}
By~\Cref{lem:dim-V_k} (and its proof), every matrix $M \in C(P_{\sigma})$ is similar to a complex block diagonal matrix $B$ 
with $\varphi(l)=\lvert (\ZZ/l\ZZ)^\times \rvert$ many blocks of size $d_l\times d_l$ for every~$l$. We will denote these blocks by $B_{l,m}$ for $m\in (\ZZ/l\ZZ)^\times$. Imposing a rank constraint on the matrix $M$  affects the rank of the diagonal blocks of~$B$. Hence, if $r_{l,m}$ is the rank of the block $B_{l,m}$, then $M$ has rank $r$ if and only if \eqref{eq:number-sol} holds. This implies that the number of different irreducible components of $\cE_{r,n\times n}^\sigma(\CC)$ is exactly the number of integer solution vectors to~\eqref{eq:number-sol}. 
\end{proof}

\begin{example}\label{ex:dimequirot}
Let $\sigma$ again denote the clockwise rotation by $90$ degrees on images with $3 \times 3$ pixels.
The numbers $d_l$ are computed in \Cref{ex:compute_dl}.
For the permutation matrix in \eqref{example:permute-99} and $r=3$, the number of irreducible components is equal to the number of non-negative integer solutions of the equation $r_{1,1}+r_{2,1}+r_{4,1}+r_{4,3} = 3$, where $r_{1,1}\leq 3=d_1$, $r_{2,1}\le 2=d_2$, and $r_{4,1},r_{4,3}\leq 2=d_4$.  With the stars and bars formula, one finds that there are $\binom{6}{3} -3 = 17$ solutions, and hence $\mathcal{E}_{3,9 \times 9}^{\sigma}(\CC)$ has $17$ irreducible components, as seen in \Cref{ex:compute_dl}. 
Six of those have dimension $11$, five have dimension $9$, and the remaining six have dimension $7$.
The six maximal-dimensional components correspond to the integer solutions $(r_{1,1}, r_{2,1}, r_{4,1}, r_{4,3}) \in \{ (2,1,0,0), \, (2,0,1,0), \, (2,0,0,1), \, (1,1,1,0), \, (1,1,0,1), \, (1,0,1,1) \}$.
\end{example}

\noindent These discussions imply that, in contrast to the case of invariance, autoencoders are  \underline{not} well-suited to parameterize \emph{all} linear equivariant functions: For a rank constraint $r<n$, $\cE_{r,n\times n}^\sigma$ has many components; the function space of an autoencoder would cover at most one of~them.

\medskip 

Not all irreducible components of $\cE^\sigma_{r, n \times n}(\CC)$ have to appear in the real locus $\cE^\sigma_{r, n \times n}(\RR)$. We will describe the real components in \Cref{sec:real-irreduce}. We now describe some algebraic properties of the complex component $\cE_{r,n\times n}^{\sigma,\mathbf{r}}(\CC)$.
\begin{proposition}
\label{prop:algebraicPropertiesComplex}
Let $\mathbf{r}$ be an integer solution of \eqref{eq:number-sol}. Then,
\begin{align}
    \begin{split}
        \dim \left(  \cE_{r,n\times n}^{\sigma,\mathbf{r}}(\CC) \right) &\,=\ 
        \sum_{l\geq 1} \sum_{m \,\in \, (\ZZ/l\ZZ)^\times} (2d_l - r_{l,m})\cdot r_{l,m} \, ,\\
        \deg \left(  \cE_{r,n\times n}^{\sigma,\mathbf{r}}(\CC) \right) &\,=\ \prod_l \prod_{m \,\in \,(\ZZ /l \ZZ)^\times} \prod_{i=0}^{d_l-r_{l,m}-1} \frac{(d_l+i)!\cdot i!}{(r_{l,m}+i)!\cdot (d_l-r_{l,m}+i)!} \, ,\\
        \Sing \left(  \cE_{r,n\times n}^{\sigma,\mathbf{r}}(\CC) \right) &\,=\ \cE_{r,n\times n}^{\sigma,\mathbf{r}}(\CC) \cap \cE^{\sigma}_{r-1, n \times n} (\CC)
        \text{ if } r < n, 
        \text{ and empty otherwise}.
    \end{split}
\end{align}
In particular, the locus of singular points of $\cE_{r,n \times n}^\sigma(\CC)$ is $\cE_{r-1,n \times n}^\sigma(\CC)$.
\end{proposition}
\begin{proof}
Due to \Cref{lem:baseChange}, we can read off the dimension, degree, and singular locus of $\cE_{r,n\times n}^{\sigma,\mathbf{r}}(\CC)$ from \eqref{eq:complexComponent}.
The first two statements follow directly from 
\eqref{eq:dimMrmn} and \eqref{eq:degMrmn}.
If $r = n$, then $\cE_{r,n \times n}^\sigma(\CC)$ is a linear space and thus smooth.
Otherwise, by \Cref{lem:sing-mat}, the singular locus of \eqref{eq:complexComponent} is 
its intersection  with $\cM_{r-1, n \times n}(\CC)$.
Now, \Cref{lem:baseChange} implies that the singular locus of $\cE_{r,n\times n}^{\sigma,\mathbf{r}}(\CC)$ is 
$\cE_{r,n\times n}^{\sigma,\mathbf{r}}(\CC) \cap \cM_{r-1, n \times n} (\CC) = \cE_{r,n\times n}^{\sigma,\mathbf{r}}(\CC) \cap \cE^{\sigma}_{r-1, n \times n} (\CC)$. 
 Finally, since the intersection of two distinct components $\cE_{r,n\times n}^{\sigma,\mathbf{r}}(\CC) $ and $\cE_{r,n\times n}^{\sigma,\mathbf{r'}}(\CC) $ is a subset of $\cE_{r-1,n \times n}^\sigma(\CC)$, we get that 
\begin{align}
\Sing \left(\cE_{r,n \times n}^\sigma (\CC) \right) \,=\,  \bigcup_{\mathbf{r}} \left(\cE_{r,n\times n}^{\sigma,\mathbf{r}}(\CC) \cap \cE_{r-1,n \times n}^\sigma(\CC)\right) \,=\ \cE_{r-1,n \times n}^\sigma (\CC) \, ,
\end{align}
 concluding the proof.
\end{proof}

\subsubsection{Irreducible components of $\cE_{r,n\times n}^\sigma(\RR)$}
\label{sec:real-irreduce} 
To obtain the real components of \( \cE_{r,n\times n}^\sigma(\RR) \), we  introduce an orthonormal basis $Q_\sigma$ in which the permutation matrix $P_\sigma$ becomes a real block diagonal matrix of a certain type. This will allow us to compute the commutator of $P_\sigma^{\sim_{Q_{\sigma}}}$ while preserving the imposed rank constraint.

\medskip

Recall that for any given circulant matrix $C\in\mathcal{M}_{n\times n}(\mathbb{R})$, the vectors 
\begin{equation}
    \label{eq:eig-circ}
    v_j \,=\, \left( 1,\zeta_n^{j},\zeta_n^{2j},\ldots,\zeta_n^{(n-1)j}\right)^{\top}\! , \qquad j=0,\ldots,n-1,
\end{equation}
are eigenvectors of~$C$, where $\zeta_n=e^{2\pi i/n}$. In the basis $\{v_0,\ldots,v_{n-1}\}$, the matrix $C$ becomes a complex diagonal matrix. 
Now, let $v_{-j} := v_{n-j} = \overline{v_j}$ and consider the following real vectors
\begin{equation}
\label{eq:base_w}  
w_0 \coloneqq \frac{1}{\sqrt{n}}v_0\, , \quad w_j \coloneqq \frac{1}{\sqrt{2n}}(v_j + v_{-j}) \, , \quad w_{-j} \coloneqq \frac{1}{\sqrt{2n}i}(v_{j} - v_{-j}) \, .
\end{equation}
The vectors $w_j$ with $-  n/2  < j \leq \lfloor n/2 \rfloor$ form a basis. We reorder them so that $w_j$ and $w_{-j}$ are next to each other. The resulting basis transforms the matrix $C$ into a real block diagonal form. Each block has size at most $2$, where scalar blocks represent the real eigenvalues, and $2 \times 2$ blocks are scaled rotation matrices of the form
\begin{align}
\begin{pmatrix}
        \Re \left(\lambda_j\right) & -\Im \left(\lambda_j\right) \\
        \Im \left(\lambda_j\right)  &  \Re \left(\lambda_j\right) 
    \end{pmatrix},
    \end{align}
    where $\lambda_j$ is a complex eigenvalue of $C$. Since  
\begin{align}
v_l^Tv_k \,=\,
\begin{cases}
    n & \text{if } l=-k \, , \\
    0 & \text{if } l \neq \pm k \, ,
\end{cases}
\end{align}
the vectors \(w_j\) form an orthonormal basis. 

Now let \(P_{\sigma}\) be a permutation matrix arranged in block form, with cyclic matrices \(C_k\) having sizes \(\ell_k \times \ell_k\) as blocks—each corresponding to a cycle of \(\sigma\). Leveraging the orthonormal basis crafted from the vectors of \eqref{eq:base_w}, we transform each \(C_k\) into a block diagonal matrix, with each block being square of size at most~\(2\). This transformation is captured by an orthonormal base change denoted by \(Q_{\sigma}^{(1)}\). By a further orthonormal base change, $Q_\sigma^{(2)}$, we group equal blocks. 
We denote the total full base change via the orthogonal matrix $Q_{\sigma}^{(2)} Q_{\sigma}^{(1)}$ by~$Q_{\sigma}$ and we call it \emph{realization base change} of $P_{\sigma}$.
\begin{example}
For the rotation of $3\times 3$ images by $90$ degrees, the permutation $\sigma\in \cS_9$ decomposes into two cycles of length $4$, and one of length $1$.  
The matrix $P_\sigma^{\sim_{T_1}}$,with a permutation matrix $T_1$ as in \Cref{proc:basechange}, has two circulant $4\times 4$ blocks, and one scalar block. As   basis for the $4\times 4$ blocks, we consider
    \begin{align}
\begin{pmatrix}
    1 & 1 & 1& 1
\end{pmatrix}^\top\! , \ \begin{pmatrix}
    1 & -1 & 1& -1
\end{pmatrix}^\top \!, \ 
\Re \begin{pmatrix}
    1 & i & -1& -i
\end{pmatrix}^\top \!, \ 
\Im  \begin{pmatrix}
    1 & i & -1& -i
\end{pmatrix}^\top 
    \end{align} 
and, after scaling to make them orthonormal, collect these vectors in the orthogonal matrix
\begin{align}
    O \,=\, {\small \frac{1}{2}\cdot \begin{pmatrix}
        1 & 1 & \sqrt{2} & 0\\
        1 & -1 & 0 & \sqrt{2}\\
        1 & 1 & -\sqrt{2} & 0\\
        1 & -1 & 0 & -\sqrt{2}
    \end{pmatrix} .}
\end{align}
A base change via the $9\times 9$ matrix \begin{align}
Q_\sigma^{(1)} = 
\begin{pmatrix}
 \begin{array}{c|c|c}
 \begin{matrix}O\end{matrix} & 0 & 0\\ \hline
 0 & O & 0 \\ \hline
 0 & 0 & 1
 \end{array}
 \end{pmatrix}
 \end{align}
brings $P_\sigma^{\sim_{T_1}}$ into the desired  form with scaled rotation matrices as blocks, namely
\begin{align}\label{eq:matrixblocksrot}
{\small 
\begin{pmatrix}
 \begin{array}{c|c|c}
\begin{matrix}
    1 &  &  & \\
     & -1 &  & \\
     &  & 0 & -1\\
     &  & 1 & 0
\end{matrix} & \text{\large 0} & 0\\ \hline
\text{\large 0} & \begin{matrix}
    1 &  &  & \\
     & -1 & & \\
     &  & 0 & -1\\
     &  & 1 & 0
\end{matrix} & 0\\ \hline 
0 & 0 & 1
 \end{array}
 \end{pmatrix}.}
 \end{align}
A further orthogonal base change $Q_\sigma^{(2)}$ via grouping identical blocks in \eqref{eq:matrixblocksrot}  brings the matrix into the block diagonal form  $\Id_3\oplus (-\Id_2) \oplus \left( \begin{smallmatrix}
     0 & -1 \\ 1 &0 
 \end{smallmatrix}\right) \oplus \left( \begin{smallmatrix}
     0 & -1 \\ 1& 0 
 \end{smallmatrix}\right) . $
 From this particularly nice form, one reads that matrices that commute with it have to be of the form 
 {
\begin{align} \label{eq:equivariantReal}
\begin{pmatrix}
    \begin{array}{c|c|c}
\begin{matrix}
\alpha_{11} & \alpha_{12} & \alpha_{13} \\
\alpha_{21} & \alpha_{22} & \alpha_{23} \\
\alpha_{31} & \alpha_{32} & \alpha_{33}
\end{matrix} &
\begin{matrix}
\text{\normalsize 0}
\end{matrix} &
\begin{matrix}
    \text{\large 0}
\end{matrix} \\ \hline
\begin{matrix}
\text{\normalsize 0}
\end{matrix} & 
\begin{matrix}
 \beta_{12} & \beta_{22}\\
 \beta_{21} & \beta_{23}
\end{matrix} &
\begin{matrix}
\text{\normalsize 0}
\end{matrix}  \\ \hline
\begin{matrix}
  \text{\large 0}
\end{matrix} & 
\begin{matrix}
  \text{\normalsize 0}
\end{matrix} &
\begin{matrix}
  \gamma_1 & -\gamma_2 & \delta_1 & -\delta_2\\
  \gamma_2 & \gamma_1 & \delta_2 & \delta_1\\
  \epsilon_1 & -\epsilon_2 & \eta_1 & -\eta_2\\
  \epsilon_2 & \epsilon_1 & \eta_2 & \eta_1
\end{matrix}
\end{array}
\end{pmatrix}.
\end{align}
}

\noindent The pattern observed here will extend to the general case.
Since the last $4 \times 4$ block in~\eqref{eq:equivariantReal} can only have rank~$0,$ $ 2$, or $4$ over $\RR$, there are five ways how the matrix \eqref{eq:equivariantReal} can be of rank~$3$. These five irreducible components of $\cE^\sigma_{3, 9 \times 9}$, that were already mentioned in \Cref{sec:warmup}, are listed in \Cref{ex:real-comp-3by3}.
\end{example}

Our next step involves finding the commutator of \(P_\sigma^{\sim_{Q_{\sigma}}}\) in general, and subsequently imposing the rank constraint.  To achieve this, we will make use of the following isomorphism of rings, to which we refer as \emph{realization}:
\begin{align}
\cR \colon \, \CC \longrightarrow \left\{\begin{pmatrix}
    a & -b \\ b & a 
\end{pmatrix} \;\Bigg|\; a,b \in \RR \right\}, \quad a+ib \, \mapsto \,  \begin{pmatrix}
        a & -b \\ b & a
\end{pmatrix} .
\end{align}
 
\begin{definition}
    Let $Z=[z_{ij}]_{ij}$ be a complex $m\times n$ matrix. We define the \emph{realization of $Z$} to be the real $2m\times 2n$ matrix $\cR(Z)$, where entry-wise application of $\cR$ is meant.  
    \end{definition}
    In other words, the realization  of $Z$ is the matrix in $\mathcal{M}_{2m \times 2n}(\RR)$ whose $(i,j)$-th block is 
    \begin{align}
        \cR(z_{ij}) \,=\, \begin{pmatrix}
        \Re \left(z_{ij}\right) & -\Im \left(z_{ij}\right) \\
        \Im \left(z_{ij}\right)  &  \Re \left(z_{ij}\right) 
    \end{pmatrix} \, .
      \end{align}
 We  denote the space of such real matrices by $\mathcal{R}(\cM_{m\times n}(\CC))$.
 The realization space $\mathcal{R}(\mathcal{M}_{r, m \times n}(\CC))$ is then precisely the set of matrices in $\mathcal{R}(\cM_{m\times n}(\CC))$ of rank at most~$2r$.
 Note that $\dim_\RR \mathcal{R}(\mathcal{M}_{r, m \times n}(\CC)) = 2 \cdot \dim_\CC \mathcal{M}_{r, m \times n}(\CC)$.

 \begin{lemma}
    \label{lem:commute-real-2block}
    Let $z \in \mathbb{C}$ be non-real.
    Then the real commutator  of the matrix $\bigoplus_{j=1}^m \mathcal{R}(z)$ is
\begin{align}
    C_{\mathbb{R}}\left(\bigoplus_{j=1}^m \mathcal{R}(z)\right) \,=\, \cR\left(\cM_{m\times m}(\CC)\right) \, .
\end{align}
\end{lemma}
\begin{proof}
    Let $M=[M_{i,j}]_{i,j}$ be a matrix in $\mathcal{M}_{2m\times 2m}(\RR)$ with $2\times 2$ blocks $M_{i,j}$. Note that $M$ commutes with $\bigoplus_{j=1}^m \mathcal{R}(z)$ if and only if every block $M_{i,j}$ commutes with $\mathcal{R}(z)$. Since $z$ is non-real, the commutator of $\mathcal{R}(z)$ is exactly the set of scaled rotation matrices. Hence, $C_{\mathbb{R}}(\bigoplus_{j=1}^m \mathcal{R}(z))$ is equal to the realization of $\cM_{m\times m}(\CC)$. 
\end{proof}

\begin{theorem}\label{thm:irreducible-dim-real}
There is a one-to-one correspondence between the irreducible components  of $\cE_{r,n\times n}^\sigma(\RR)$ that contain a matrix of rank~$r$ and the non-negative integer solutions~$\mathbf{r}=(r_{l,m})$~of
\begin{equation}\label{eq:number-sol-real}
r_{1,1}+r_{2,1}+  \sum_{l \geq 3} \sum_{\substack{m \,\in\, (\ZZ /l \ZZ)^\times, \\ \frac{1}{2}< \frac{m}{l} < 1}} \hspace*{-2mm} 2\cdot r_{l,m} \ =\ r \, ,\quad  \text{ where } \ 0 \le r_{l,m} \le d_l\, .
\end{equation}  
The irreducible component $\cE_{r,n\times n}^{\sigma,\mathbf{r}}(\RR)$ corresponding to such an integer solution $\mathbf{r}$ after the base change $Q_\sigma$ is
\begin{align}\label{eq:irrcompequi}
 \mathcal{M}_{r_{1,1}, d_1 \times d_1}(\RR)
 \, \times \,
  \mathcal{M}_{r_{2,1}, d_2 \times d_2}(\RR)
 \, \times \,
\prod_{l \geq 3} \prod_{\substack{m \,\in\, (\ZZ /l \ZZ)^\times, \\ \frac{1}{2}< \frac{m}{l} < 1}} \hspace*{-3mm} \mathcal{R}(\mathcal{M}_{r_{l,m},d_l \times d_l}(\CC)) \,.
\end{align}
\end{theorem}
\begin{proof}

By \Cref{lem:invAndEquivUnderbaseChange}, 
we have that $C_{\RR}(P_{\sigma}^{\sim_{Q_{\sigma}}})=(C_{\RR}(P_{\sigma}))^{\sim_{Q_{\sigma}}}  = (\cE^\sigma_{n \times n}(\RR))^{\sim_{Q_{\sigma}}}$.
Also, \Cref{lem:commute-block} implies that the commutator of $P_{\sigma}^{\sim_{Q_{\sigma}}}$ is the set of block diagonal matrices $R = \bigoplus_{m, \, l} R_{l,m}$, where the direct sum is running over $ m \,\in\, (\ZZ /l \ZZ)^\times$ and {$\{l \,|\, \frac{1}{2}\le \frac{m}{l} \le 1 \}$},
 such that $R_{1,1}$ and $R_{2,1}$ are arbitrary matrices of size $d_1\times d_1$ and $d_2\times d_2$, respectively. For the other pairs $(l,m)$ with $\frac{1}{2}<\frac{m}{l}<1$ and $\gcd(l,m)=1$, \Cref{lem:commute-real-2block} implies that the matrix  $R_{l,m}$ is an arbitrary matrix in $\mathcal{R}(\cM_{d_l \times d_l}(\CC))$. Now, note that for a complex matrix $Z$, we have $\rank_\RR (\cR(Z))={2 \cdot }\rank_{\CC}(Z)$. Hence, imposing the rank constraint leads to \eqref{eq:number-sol-real}, and the corresponding irreducible component can be seen in \eqref{eq:irrcompequi}.
\end{proof}

\begin{example}
\label{ex:real-comp-3by3}
For rotating $3\times 3$ images by $90$ degrees, the permutation $\sigma\in \cS_9$ decomposes into two cycles of length $4$, and one of length $1$. Then $P_\sigma^{\sim_{Q_\sigma}}$ is the block diagonal matrix 
$\Id_3 \oplus (-\Id_2) \oplus \left [\mathcal{R}(i) \oplus \mathcal{R}(i)\right]$. Thus, the commutator of $P_\sigma^{\sim_{Q_\sigma}}$ 
is equal to $\cM_{3\times 3}(\RR) \oplus \cM_{2\times 2}(\RR) \oplus \mathcal{R}(\cM_{2\times2}(\CC))$. The matrix in~\eqref{eq:equivariantReal} is an example of an element of that space. For  $r=3$, the number of real irreducible components matches the non-negative integer solutions of the equation $r_{1,1}+r_{2,1}+2\cdot r_{4,3} = 3$. Here, $r_{1,1}\leq 3=d_1$, $r_{2,1}\le 2=d_2$, and $r_{4,3}\leq 2=d_4$. Solving the equation yields five solutions: $(3,0,0),(2,1,0),(1,2,0),(1,0,1),$ and $(0,1,1)$. This is a notable 
decrease of the number of complex irreducible components from~$17$ in \Cref{ex:dimequirot} to~$5$. The dimensions of the $5$ real components are $9,$ $11,$ 
 $9$, $11$, and~$9$.
\end{example}

We can compute the dimensions of the real components of $\cE^\sigma_{r, n \times n}$ in a similar fashion as the complex components; see \Cref{prop:algebraicPropertiesReal}.
For the singular loci of the real components, we make use of the following observation.
Recall that for a given holomorphic function $f:\CC^n \to \CC$, we can write
\begin{equation}
    f\left(z_1,\ldots,z_n\right)  \,=\, u\left(x_1,y_1,\ldots,x_n,y_n\right) \,+\, i\, v\left(x_1,y_1,\ldots,x_n,y_n\right),
\end{equation}
where $u$ and $v$ are smooth real-valued functions and $z_j = x_j+iy_j$ for every $j$. The Cauchy--Riemann equations and Wirtinger derivatives provide  the following relations:
\begin{align}\label{eq:cauchy-riem}
    \frac{\partial u}{\partial x_{j}} = \frac{\partial v}{\partial y_{j}} \quad \text{and} \quad \frac{\partial u}{\partial y_{j}} = -\frac{\partial v}{\partial x_{j}} \, , 
    \end{align}
    and
\begin{align}\label{eq:Wirtinger}
    \frac{\partial f}{\partial z_j} \,=\, \frac{1}{2}\left(\frac{\partial f}{\partial x_j} - i \frac{\partial f}{\partial y_j} \right) \,=\, \frac{\partial u}{\partial x_j} + i \frac{\partial v}{\partial x_j} 
\end{align}
for all $j=1,\ldots,n$. Therefore, by~\eqref{eq:cauchy-riem} and~\eqref{eq:Wirtinger}, the Jacobian $\operatorname{Jac}(u,v)=\left[\begin{smallmatrix}
\nabla u \\
  \nabla v
\end{smallmatrix}\right]$ can be simplified as follows:
\begin{align}
\label{eq:cauchy-riemann-realziation}
\begin{aligned}
\operatorname{Jac}(u,v)&\,=\,\begin{bmatrix}
    \frac{\partial u}{\partial x_1} & \frac{-\partial v}{\partial x_1}& \cdots & \frac{\partial u}{\partial x_n} & \frac{-\partial v}{\partial x_n}\\
    \frac{\partial v}{\partial x_1} & \frac{\partial u}{\partial x_1}& \cdots & \frac{\partial v}{\partial x_n} & \frac{\partial u}{\partial x_n}
\end{bmatrix}\,=\,\cR \left[\frac{\partial u}{\partial x_1}+i \frac{\partial v}{\partial x_1},\ldots,\frac{\partial u}{\partial x_n}+i \frac{\partial v}{\partial x_n} \right]\\
&\,=\, \cR \left( \left [ 
\frac{\partial f}{\partial x_1},\ldots,\frac{\partial f}{\partial x_n}  \right]\right)=\cR \left( \left [ 
\frac{\partial f}{\partial z_1},\ldots,\frac{\partial f}{\partial z_n}  \right] \right)\,=\, \cR \left(\nabla f\right).\end{aligned}\end{align}
\begin{lemma}
    \label{lem:rank-jac-realization}
    Let $I \subset \CC[z_1,\ldots,z_n]$ be a prime ideal generated by polynomials $f_1,\ldots,f_k$. {We split $f_j=u_j + iv_j$ into their real and imaginary parts and denote by $J \subset \RR[x_1,y_1,\ldots,x_n,y_n]$ the ideal generated by the $u_j$'s and $v_j$'s.}  Then a point $\mathbf{z}_{0}\in \CC^n$ is singular for the complex variety $V(I)$ if and only if {$(\mathbf{x}_0,\mathbf{y}_0)=(\Re(\mathbf{z}_0),\Im(\mathbf{z}_0))$} in $\RR^{2n}$ is singular for~$V(J)$.
\end{lemma}
\begin{proof}
    By \eqref{eq:cauchy-riemann-realziation}, for any $\mathbf{z}_0 \in \CC^n$ and its counterpart $(\mathbf{x}_0,\mathbf{y}_0)\in \RR^{2n}$, we have
\begin{align}\begin{split}
\label{eq:rank-comp-real}
\rank_{\RR}\operatorname{Jac}_{(\mathbf{x}_0,\mathbf{y}_0)}\left(u_1,v_1,\ldots,u_k,v_k\right) &\,=\, \rank_{\RR}\cR\left(\operatorname{Jac}_{\mathbf{z}_0}\left(f_1,\ldots,f_k\right)\right) \\ &\,=\, 2\cdot \rank_{\CC}\operatorname{Jac}_{\mathbf{z}_0}\left(f_1,\ldots,f_k\right).
\end{split}\end{align}
Note that a point $\mathbf{z}_0 \in \CC^n$ (resp., $(\mathbf{x}_0,\mathbf{y}_0) \in \RR^{2n}$) is singular for $V(I)$ (resp., $V(J)$) if and only if the rank of the Jacobian $\operatorname{Jac}_{\mathbf{z}_0}(f_1,\ldots,f_k)$ (resp., $\operatorname{Jac}_{(\mathbf{x}_0,\mathbf{y}_0)}(u_1,v_1,\ldots,u_k,v_k)$) drops.
Hence, due to \eqref{eq:rank-comp-real}, $\mathbf{z}_0$ is singular for $V(I)$ if and only if $(\mathbf{x}_0,\mathbf{y}_0)$ is singular for~$V(J)$.
\end{proof}
\begin{corollary}\label{cor:singularOfRealization}
Let $0 < r < \min(m,n)$.
A matrix $M \in \cR(\cM_{r, m \times n}(\CC))$ is a singular point of $\cR(\cM_{r, m \times n}(\CC))$ if and only if $M \in \cR(\cM_{r-1, m \times n}(\CC)).$
\end{corollary}
\begin{proof}
    Let $I = \langle f_1,\ldots, f_k \rangle$ be the ideal of $\cM_{r,m\times n}(\CC)$. Then the realization $\cR({\cM_{r,m\times n}(\CC)})$ is the {common} zero locus of $u_1,v_1,\ldots,u_k,v_k$, where $f_j = u_j+iv_j$. By \Cref{lem:sing-mat}, we know that $\Sing(\cM_{r,m\times n}(\CC))= \cM_{r-1,m\times n}(\CC)$, and thus \Cref{lem:rank-jac-realization} implies that $\Sing(\cR(\cM_{r,m\times n}(\CC)))= \cR(\cM_{r-1,m\times n}(\CC))$.
\end{proof}
\begin{proposition}\label{prop:algebraicPropertiesReal}
Let $\mathbf{r}$ be an integer solution of \eqref{eq:number-sol-real}. The dimension of $\cE_{r,n\times n}^{\sigma,\mathbf{r}}(\RR) $ is
\begin{align}
          (2d_1 - r_{1,1})\cdot r_{1,1} \,+\,(2d_2 - r_{2,1})\cdot r_{2,1} 
          \,+\, {{2\cdot}} \sum_{l\geq 3} \sum_{\substack{m \,\in \, (\ZZ/l\ZZ)^\times, \\ \frac{1}{2}< \frac{m}{l} < 1}} (2d_l - r_{l,m})\cdot r_{l,m}  \, ,
          \end{align}
          and its singular locus is
          \begin{align}
        \Sing \left(  \cE_{r,n\times n}^{\sigma,\mathbf{r}}(\RR) \right) &\,=\ \cE_{r,n\times n}^{\sigma,\mathbf{r}}(\RR) \cap \cE^{\sigma}_{r-1, n \times n} (\RR)
        \text{ if } r < n, \text{ and empty otherwise}.
\end{align}
\end{proposition}
\begin{proof}
As in the proof of \Cref{prop:algebraicPropertiesComplex}, we make use of \Cref{lem:baseChange}.
    Then, the first statement follows from 
    \eqref{eq:irrcompequi},
    the fact that $\dim_\RR \mathcal{R}(\mathcal{M}_{r, m \times n}(\CC)) = 2 \cdot \dim_\CC \mathcal{M}_{r, m \times n}(\CC)$, and~\eqref{eq:dimMrmn}.
If $r = n$, then $\cE_{r,n\times n}^{\sigma,\mathbf{r}}(\RR)$ is a linear space and thus smooth.
Otherwise, by \Cref{lem:baseChange} and \Cref{cor:singularOfRealization}, the singular locus of \eqref{eq:irrcompequi} is its intersection with $\cM_{r-1,n \times n}$.
Now, the second assertion follows from \Cref{lem:baseChange}.
\end{proof}
In light of  \Cref{sec:parameterizingEquiv}, where we will show that each real irreducible component $\cE_{r,n\times n}^{\sigma,\mathbf{r}}(\RR) $ is the function space of a $\sigma$-equivariant autoencoder, the second assertion of \Cref{prop:algebraicPropertiesReal} means that the singular locus of the function space is the finite union of function spaces of {networks with} smaller architectures.

\medskip

To a give a degree formula for the real components of $\cE^\sigma_{r, n \times n}$ as we did in the complex case in \Cref{prop:algebraicPropertiesComplex}, we would need to have a formula for the degree of the realization space $\mathcal{R}(\mathcal{M}_{r, m \times n}(\CC)) \subset \RR^{2m \times 2n}$.
More precisely, for a well-defined notion of degree, we need to consider the degree of the Zariski closure of the real variety $\mathcal{R}(\mathcal{M}_{r, m \times n}(\CC))$ inside $\CC^{2m \times 2n}$.
For that degree, we conjecture the following:
\begin{conjecture}\label{conj:degsq}
$\deg \cR(\cM_{r, m \times n}) = (\deg \cM_{r, m \times n})^2$.
\end{conjecture}

We validated the conjecture for all triples $(m,n,r)$ with $1 \leq m \leq 3, 1 \leq n \leq 9, 1 \leq r \leq 8$, or $m=4, 4 \leq n \leq 5, 2 \leq r \leq 5$, or $(m,n,r) \in \{ (5,4,2), \, (5,5,1), \, (6,6,1), \, (6,7,1), \, (7,6,1) \}$.
To test \Cref{conj:degsq} for a specific choice of parameters $(m,n,r)$, one can run the following lines in~{\tt Macaulay2}, 
here displayed for $(m,n,r)=(3,2,1)$:

{\footnotesize
\begin{verbatim}
m = 3; n = 2; r = 1;
R1 = QQ[c_(1,1)..c_(m,n)]; 
M = matrix apply(toList(1..m), i -> apply(toList(1..n), j -> c_(i,j)));
I = minors(r+1,M); 
R2 = QQ[a_(1,1)..a_(m,n),b_(1,1)..b_(m,n), x] / ideal(x^2+1);
f = map(R2, R1, flatten apply(toList(1..m), 
    i -> apply(toList(1..n), j -> a_(i,j) + x*b_(i,j))));
getRealAndImaginaryPart = eq -> (
    eqReal = sub(eq, x=>0);
    eqImag = sub((eq - eqReal)/x,R2);
    {eqReal,eqImag}
    );
J = ideal flatten apply(flatten entries gens I, 
    eq -> getRealAndImaginaryPart f eq); 
R3 = QQ[a_(1,1)..a_(m,n),b_(1,1)..b_(m,n)];
degree sub(J,R3) == (degree I)^2
\end{verbatim}
}

\subsection{Squared-error loss minimization on $\cE_{r,n\times n}^\sigma(\RR)$}\label{sec:EDopt}
The explicit structure of the space of equivariant linear maps in \Cref{thm:irreducible-dim-real}, with a bound on their rank imposed, provides an efficient algorithm to find the point $M$ in $\cE_{r,n\times n}^\sigma(\RR)$ that minimizes the squared-error loss.
Given $X \in \RR^{n \times d}$ with $\rank(XX^\top)=n$ and $Y \in \RR^{m \times d}$,
our task is to find a point in $\cE_{r,n\times n}^\sigma(\RR)$ that minimizes $\Vert MX-Y \Vert^2_{F}$. {For autoencoders, the input data equals the output data, i.e., $Y=X$, but we here allow arbitrary output data~$Y$.} 
The following algorithm reduces this task to many instances of minimizing the standard Euclidean distance to rank-bounded matrices via Eckart--Young, which in particular  shows $\squerdeg(\cE_{r,n\times n}^\sigma(\RR)) = \deged(\cE_{r,n\times n}^\sigma(\RR))$.
We proceed with the following three steps, for each of  which we give further details right after.

\begin{enumerate}[Step~1.]
    \item {\em Transform {the  task to finding a block diagonal matrix $B \in (\cE_{r,n\times n}^\sigma(\RR))^{\sim Q_\sigma}$ that minimizes $\Vert B - U \Vert^2_{\tilde{X}\tilde{X}^\top}$, where $\tilde X \coloneqq Q_\sigma^\top X$ and $U \coloneqq Q_\sigma^\top Y \tilde{X}^\top (\tilde{X}\tilde{X}^\top)^{-1}$}.}
    \end{enumerate}
Due to the orthogonality of $Q_\sigma$, we see as in~\eqref{eq:EDproblemOrthTransform} that 
$\Vert MX - Y \Vert^2_F = \Vert Q_\sigma^\top M X - Q_\sigma^\top Y \Vert_F^2 = \Vert M^{\sim Q_\sigma} \tilde{X} - \tilde{Y} \Vert_F^2$, where $\tilde{Y} \coloneqq  Q_\sigma^\top Y$. 
 Hence,  $M$  in $\cE_{r,n\times n}^\sigma(\RR)$ minimizes the squared-error loss with data matrices $X,Y$ if and only if the block diagonal matrix $B := M^{\sim Q_\sigma}$ in $(\cE_{r,n\times n}^\sigma(\RR))^{\sim Q_\sigma}$ minimizes the squared-error loss with data matrices $\tilde{X},\tilde{Y}$. Equivalently, by \Cref{lem:squaredErrorED}, $B$ minimizes $\Vert B - U \Vert^2_{\tilde{X}\tilde{X}^\top}$.
\begin{enumerate}[Step~2.]
    \item {\em With respect to the inner product $\langle \cdot, \cdot \rangle_{\tilde{X}\tilde{X}^\top}$, compute the orthogonal projection~$\tilde U$ of~$U$ onto the linear space $(\cE_{n\times n}^\sigma(\RR))^{\sim Q_\sigma}$.}
    \end{enumerate}
    \noindent 
    {Since $\Vert B - U \Vert^2_{\tilde{X}\tilde{X}^\top} = \Vert B - \tilde{U} \Vert^2_{\tilde{X}\tilde{X}^\top} + \Vert \tilde{U} - U \Vert^2_{\tilde{X}\tilde{X}^\top}$,}
    the point on the variety $(\cE_{r,n\times n}^\sigma(\RR))^{\sim Q_\sigma}$ closest {(w.r.t. $\Vert \cdot \Vert_{\tilde{X}\tilde{X}^\top}$)} to either $\tilde U$ or $U$ is the same.
    The matrices in the linear space $(\cE_{n\times n}^\sigma(\RR))^{\sim Q_\sigma}$, including $\tilde U$, are block diagonal matrices, whose blocks are either in $\cM_{d_l \times d_l}(\RR)$ or  $\mathcal{R}(\cM_{d_l \times d_l}(\CC))$.
    {Using \Cref{lem:directProduct}, we can solve the $\Vert \cdot \Vert_{\tilde{X}\tilde{X}^\top}$-distance} problem from $\tilde U$ on each block separately. 
    Since the variety $(\cE_{r,n\times n}^\sigma(\RR))^{\sim Q_\sigma}$ has several irreducible components, we can find its point closest to~$\tilde U$ by solving the minimization problem on each component individually. Hence:
    \begin{enumerate}[Step~3.]
    \item {\em On each irreducible component  $(\cE_{r,n\times n}^{\sigma,\textbf{r}}(\RR))^{\sim Q_\sigma}$ (described in \Cref{thm:irreducible-dim-real}) and on each diagonal matrix block (i.e., on each factor of the direct product~\eqref{eq:irrcompequi}), resp., solve
\begin{align}\label{eq:subproblems}
        \argmin_{B_l \,\in\, \mathcal{M}_{r_{l,m},d_l \times d_l}(\RR)} \Vert B_l - \tilde{U}_l \Vert^2_{{\tilde{X}_l\tilde{X}_l^\top}}
        \quad \ \text{and} \  \quad
        \argmin_{B_l \,\in\, \mathcal{R}(\mathcal{M}_{r_{l,m},d_l \times d_l}(\CC))} \Vert B_l - \tilde{U}_l \Vert^2_{{\tilde{X}_l\tilde{X}_l^\top}}, 
     \end{align}
     respectively,
     using Eckart--Young (see \Cref{lem:EYrealization}), 
     where $\tilde{U}_l$ denote the blocks of $\tilde U$, {and $\tilde{X}_l$ consists of the corresponding rows of $\tilde{X}$.}}
     \end{enumerate}

\noindent Writing $B_l^{\textbf{r}}$ for the solutions of these subproblems, 
     we consider the block diagonal matrices $\bigoplus_{l} B_l^{\textbf{r}}$, one for each irreducible component indexed by $\textbf{r}$.
     Out of these finitely many matrices, the one that is $\Vert \cdot \Vert_{\tilde{X}\tilde{X}^\top}$-closest to $\tilde U$ is the matrix $B$ from Step~1.
     Hence, $Q_\sigma B Q_\sigma^\top $ is a  $\sigma$-equivariant matrix of rank at most $r$ that minimizes the squared-error loss with data matrices~$X,Y$.

    \medskip
     The only ingredient in this algorithm that is missing an explanation, is how to minimize the squared-error loss in \eqref{eq:subproblems} on spaces of realization matrices. The following lemma shows that we can reduce this problem to spaces of matrices without any special structure {imposed}.
\begin{lemma}
\label{lem:SELonRealization}
    Let $A \in \cR(\cM_{m \times n}(\CC))$ and let $T \in \RR^{n \times n}$ be a positive definite matrix. Then
    \begin{align}
        \Vert A \Vert_T^2 \ =\, \Vert A_{\mathrm{odd}} \Vert^2_{T+PTP^\top} \, ,
    \end{align}
    where $A_{\mathrm{odd}} \in \cM_{m \times 2n}(\RR)$ consists of all odd rows of $A$ and $P \coloneqq \left( \begin{smallmatrix}
        0 & 1 \\ -1 & 0
    \end{smallmatrix} \right) \oplus \cdots \oplus \left( \begin{smallmatrix}
        0 & 1 \\ -1 & 0
    \end{smallmatrix} \right) \in \RR^{2n \times 2n}$.
\end{lemma}
\begin{proof}
    Let $A_{\mathrm{even}}$ be the $m \times 2n$ matrix that consists of all even rows of $A$.
    Since $A$ is a realization matrix, we have that $A_{\mathrm{even}} = A_{\mathrm{odd}} \cdot P$.
    Therefore, 
\begin{align}
    \begin{split}
            \Vert A \Vert_T^2 \,&=\, \Vert A T^{1/2} \Vert^2_F
   \,=\, \Vert A_{\mathrm{odd}} T^{1/2} \Vert^2_F \,+\, \Vert A_{\mathrm{even}} T^{1/2} \Vert^2_F
      \,=\, \Vert A_{\mathrm{odd}} T^{1/2} \Vert^2_F \,+\, \Vert A_{\mathrm{odd}} P T^{1/2} \Vert^2_F \\
      &=\, \mathrm{tr}\left(A_{\mathrm{odd}}T A_{\mathrm{odd}}^\top\right) \,+\, \mathrm{tr}\left(A_{\mathrm{odd}}PTP^\top A_{\mathrm{odd}}^\top\right)
      \,=\, \mathrm{tr}\left(A_{\mathrm{odd}} \left(T + PTP^\top\right) A_{\mathrm{odd}}^\top\right) \\
      &=\, \Vert A_{\mathrm{odd}} \Vert^2_{T+PTP^\top},
    \end{split}
\end{align}
concluding the proof.
\end{proof}
     
\begin{proposition}
\label{lem:EYrealization}
    Both minimization problems in \eqref{eq:subproblems} can be solved with Eckart--Young and have $\squerdeg = \deged = \binom{d_l}{r_{l,m}}$.
\end{proposition}
\begin{proof}
    We use \Cref{ex:FisEverything} for blocks of the form $\mathcal{M}_{r_{l,m},d_l \times d_l}(\RR)$. For blocks of the form $\mathcal{R}(\mathcal{M}_{r_{l,m},d_l \times d_l}(\CC))$, 
    we consider the orthogonal projection $\tilde{U}^\circ_l$ from $\tilde{U}_l$ {(with respect to the inner product $\langle \cdot,\cdot \rangle_{\tilde{X}\tilde{X}^\top}$)} onto the linear space $\mathcal{R}(\mathcal{M}_{d_l \times d_l}(\CC))$.
    The same point on the variety $\mathcal{R}(\mathcal{M}_{r_{l,m},d_l \times d_l}(\CC))$ minimizes the {$\Vert \cdot \Vert_{\tilde{X}\tilde{X}^\top}$-distance} to  $\tilde{U}_l$ and $\tilde{U}^\circ_l$.
    Now, we delete {all the even rows} in the orthogonal projection $\tilde{U}^\circ_l$ and denote the result by $\tilde{U}_l^{\mathrm{odd}} \in \cM_{d_l \times 2d_l}(\RR)$.
    Deleting the same rows in $B_l \in \mathcal{R}(\mathcal{M}_{r_{l,m},d_l \times d_l}(\CC))$ gives arbitrary $d_l \times 2d_l$ matrices $B_l^{\mathrm{odd}}$ of rank at most $r_{l,m}$.
 {By \Cref{lem:SELonRealization}, 
    we have $$\Vert B_l  - \tilde{U}^\circ_l \Vert^2_{\tilde{X}\tilde{X}^\top} \,=\,  \Vert B_l^{\mathrm{odd}} - \tilde{U}_l^{\mathrm{odd}} \Vert^2_{S}\, ,$$ where $S \coloneqq \tilde{X}\tilde{X}^\top + P \tilde{X}\tilde{X}^\top P^\top$.}
    Hence, the desired minimizer is the matrix $B_l \in \mathcal{R}(\mathcal{M}_{d_l \times d_l}(\CC))$ such that $B_l^{\mathrm{odd}}$ is the matrix in $ \cM_{r_{l,m}, d_l \times 2d_l}(\RR)$ minimizing $\Vert B_l^{\mathrm{odd}} - \tilde{U}_l^{\mathrm{odd}} \Vert^2_{S}$.
    The latter can be solved with Eckart--Young~\eqref{eq:EYminimum} as in \Cref{ex:FisEverything}.
    We also see from \eqref{eq:EDdegMrmn} that both the squared-error degree and the ED degree of this problem are $\binom{d_l}{r_{l,m}}$.
\end{proof}

\pagebreak
\subsection{Parameterizing equivariance and network design} \label{sec:parameterizingEquiv}
One observes even in simple examples that the factors of an equivariant linear map themselves do \emph{not} need to be equivariant. 

\begin{example}
Let $\sigma=(1 \, 2)\in \cS_3$ and $M$ be the invertible matrix
\begin{align*}
M \,=\, 
\begin{pmatrix}
1 & 2 & 0\\
2 & 1 & 0\\
3 & 3 & 4
\end{pmatrix} .
\end{align*}
Indeed, $MP_\sigma=P_\sigma M$, hence $M$ is equivariant under $\sigma$. Let $M=QR$ denote the QR decomposition of~$M$; uniqueness of the decomposition is obtained by imposing that~$R$ has positive diagonal entries. One can check that neither~$Q$ nor~$R$ is equivariant under~$\sigma$.
\end{example}

\begin{remark} 
The question whether the individual layers of an equivariant autoencoder are equivariant, is not well-posed in its na\"{i}ve form. A priori, the group~$G$ acts only on the in- and output space of $f_\theta\colon \, \RR^n\to \RR^r \to \RR^n$. To address questions about equivariance of the two individual layers, one would first need to define an action of $G$ on~$\RR^r$.
\end{remark}

In \Cref{sec:networkDesign}, we described how linear autoencoders are well-suited to parameterize permutation-invariant maps. 
For equivariance, auto-encoders can only parameterize the individual irreducible components of $\cE_{r,n\times n}^\sigma$. Also in this case, the decoder and encoder inherit a weight-sharing property from the cycle decomposition of~$\sigma$. 
To develop an intuition, we start with an example for rotation-equivariant maps of rank at most~$1$.

\begin{example}[Parameterization of $\cE^\sigma_{1,9\times 9}$]\label{ex:paramE199}
Let $\sigma\in \cS_9$ again denote the rotation of a $3\times 3$ picture by $90$ degrees. Denote by $P$ the matrix obtained by applying Step $1$ of \Cref{proc:basechange} to $P_\sigma$, i.e., $P$ is the block diagonal matrix $\operatorname{diag}(C_4,C_4,C_1)$. Its eigenvalues are $\{1,1,1,-1,-1,i,i,-i,-i\}$, here denoted as multiset together with their multiplicities. 
We chop $M$ into blocks of sizes determined by the blocks of~$P$, i.e. into blocks of size pattern
\begin{align*}
\begin{pmatrix}
    \begin{array}{c|c|c}
\begin{matrix}
4 \times 4
\end{matrix} & \begin{matrix}
4 \times 4
\end{matrix} & \begin{matrix}
4 \times 1 
\end{matrix}\\ \hline
\begin{matrix}
4 \times 4
\end{matrix} & \begin{matrix}
4 \times 4
\end{matrix} & \begin{matrix}
4 \times 1 
\end{matrix}\\ \hline
 \begin{matrix}
1 \times 4
\end{matrix} & \begin{matrix}
1 \times 4
\end{matrix} & \begin{matrix}
1 \times 1 
\end{matrix}
\end{array}
\end{pmatrix}.
\end{align*}
Each of the blocks $M^{(i,j)}$ is circulant, as spelled out in \eqref{eq:matrotequi}. We are now going to describe the set of matrices $M$ which commutate with $P$ and are of rank at most~$1$.
For that, we will need circulant matrices. For a vector $v=(v_1,\ldots,v_n)^\top\in \CC^n$, we denote the associated $n\times n$ circulant matrix by
\begin{align}\label{eq:Toep}
    C_n(v_1,\ldots,v_n) \, \coloneqq \, \begin{pmatrix}
        v_1 & v_2 & \cdots &v_{n-1} & v_n \\
        v_n & v_1 & v_2 & \cdots & v_{n-1} \\
         & & \ddots  & \ddots &\\
         v_3 &\cdots & v_n & v_1 & v_2\\
        v_2 & v_3 &\cdots & v_n & v_1
    \end{pmatrix} \, \in \, \cM_{n\times n}(\CC) \, .
\end{align}
In this notation, the circulant matrices $C_n$~in~\eqref{eq:circulant} are $C_n=C_n(0,\ldots,0,1).$
Imposing $r=1$ gives rise to three irreducible components of $\cE^\sigma_{1,9 \times 9}(\CC)$ of dimension $3$, and one of dimension~$5$, according to \Cref{thm:irreducible-dim}. By $\mathbbm{1}$, we will denote the all-one matrix; its size is determined implicitly by the rest of the matrix.
An explicit analysis reveals that the general matrices in the four components of  $\cE_{1,9\times 9}^\sigma(\CC)$ take the following forms for 
scalars $\alpha_1, \alpha_2, \alpha_3, \beta_1, \beta_2, \beta_3\in \CC$: 

{\small 
\begin{align}\begin{split}
\begin{pmatrix}
    \begin{array}{c|c|c}
\alpha_1 \beta_1 C_4(1,1,1,1)
 &  \alpha_1\beta_2 C_4(1,1,1,1) &
\begin{matrix}
 \alpha_1 \beta_3 \cdot \mathbbm{1}
\end{matrix} \\ \hline
 \alpha_2 \beta_1 C_4(1,1,1,1)
 &  \alpha_2 \beta_2 C_4(1,1,1,1) &
\begin{matrix}
 \alpha_2\beta_3 \cdot \mathbbm{1}
\end{matrix} \\ \hline
\alpha_3 \beta_1 \cdot  \mathbbm{1}
& 
 \alpha_3\beta_2  \cdot \mathbbm{1}
&
 \alpha_3\beta_3
\end{array} 
\end{pmatrix}, \\ \begin{pmatrix}
    \begin{array}{c|c|c}
\alpha_1 \beta_1 C_4(1,-1,1,-1)
 &  \alpha_1 \beta_2 C_4(1,-1,1,-1) &
\begin{matrix}
0
\end{matrix} \\ \hline
 \alpha_2 \beta_1  C_4(1,-1,1,-1)
 &  \alpha_2  \beta_2 C_4(1,-1,1,-1) &
\begin{matrix}
0
\end{matrix} \\ \hline
0 & 
0 &
0
\end{array}
\end{pmatrix},\\ 
\begin{pmatrix}
    \begin{array}{c|c|c}
\alpha_1 \beta_1 C_4(1,-i,-1,i)
 &  \alpha_1 \beta_2 C_4(1,-i,-1,i) &
\begin{matrix}
0
\end{matrix} \\ \hline
 \alpha_2 \beta_1  C_4(1,-i,-1,i)
 &  \alpha_2 \beta_2 C_4(1,-i,-1,i) &
\begin{matrix}
0
\end{matrix} \\ \hline
0 & 
0 &
0
\end{array}
\end{pmatrix}, \\ 
\begin{pmatrix}
    \begin{array}{c|c|c}
\alpha_1  \beta_1 C_4(1,i,-1,-i)
 &  \alpha_1 \beta_2 C_4(1,i,-1,-i) &
\begin{matrix}
0
\end{matrix} \\ \hline
 \alpha_2  \beta_1 C_4(1,i,-1,-i)
 &  \alpha_2 \beta_2 C_4(1,i,-1,-i) &
\begin{matrix}
0
\end{matrix} \\ \hline
0 & 
0 &
0
\end{array}
\end{pmatrix}.
\end{split}
\end{align}
}

\noindent The first component is isomorphic to the affine cone over the Segre variety $\PP_\CC^2 \times \PP_\CC^2$, the remaining three are isomorphic to the affine cone over the Segre variety $\PP_\CC^1 \times \PP_\CC^1$ each. Only the first two of them appear in the real locus~$\cE_{1,9\times 9}^\sigma(\RR)$; cf. \Cref{thm:irreducible-dim-real}. The matrices in those two components can be factorized as follows:
\begin{align}
    \begin{split}
(\alpha_1,\alpha_1,\alpha_1,\alpha_1,\alpha_2,\alpha_2,\alpha_2,\alpha_2,\alpha_3)^\top \cdot (\beta_1,\beta_1,\beta_1,\beta_1,\beta_2,\beta_2,\beta_2,\beta_2,\beta_3) \, , \hspace*{11.2mm} \\
        (\alpha_1,-\alpha_1,\alpha_1,-\alpha_1,\alpha_2,-\alpha_2,\alpha_2,-\alpha_2,0)^\top \cdot (\beta_1,-\beta_1,\beta_1,-\beta_1,\beta_2,-\beta_2,\beta_2,-\beta_2,0) \, .
    \end{split}
\end{align}
These factorizations are linear autoencoders $\RR^9 \to \RR^1 \to \RR^9$ with the same weight-sharing on the en- and decoder.
\end{example}

In \Cref{sec:real-irreduce}, we characterized the real irreducible components in general. We here present a parameterization for each of the irreducible components described in \Cref{thm:irreducible-dim-real}. 
For that, we consider the following parameterization of $\cM_{r,n\times n}(\CC)$:
\begin{align}
\mu_{r,n}^{\CC}\colon \ \mathcal{M}_{n \times r}(\CC) \times \mathcal{M}_{r \times n}(\CC)\longrightarrow \mathcal{M}_{r,n \times n}(\CC),\quad (A,B)\, \mapsto \,A\cdot B \, .
\end{align}
Hence, we will write every element of $\cM_{r,n\times n}(\CC)$ as a product of two matrices~$A$ and~$B$; they will play the role of a decoder and encoder, respectively.
Since $\cR$ is an  isomorphism of rings, the complex matrix-multiplication map extends to realization spaces:
\begin{align}
\begin{split}\label{eq:para-realization}
\mathcal{R}\mu_{r,n}^\CC\colon \  \mathcal{R}(\cM_{n\times r}(\CC)) \times \mathcal{R}(\cM_{r\times n}(\CC)) & \,\to\, \mathcal{R}(\cM_{r,n \times n}(\CC)), \\ (\cR(A),\cR(B)) & \,\mapsto\,  \cR(A) \cdot \cR(B) \,=  \ \cR(A\cdot B) \, . 
\end{split}
\end{align}

\begin{proposition}
\label{prop:real-param-equi}
Let $\mu_{r_{1,1},d_1}$ and $\mu_{r_{2,1},d_2}$ denote the real parameterization maps from \Cref{prop:fiberAB} for $\mathcal{M}_{r_{1,1},d_1\times d_1}(\RR)$ and $\mathcal{M}_{r_{2,1},d_2\times d_2}(\RR)$, respectively. Let $\mu_{r_{l,m},d_l}^{\CC}$ be the complex parameterization of $\mathcal{M}_{r_{l,m},d_l \times d_l}(\CC)$, for $l\ge  3$. Then the real irreducible component of $(\cE_{r,n\times n}^{\sigma})^{\sim Q_{\sigma}}$ corresponding to $\mathbf{r} = (r_{l,m})$ as in \Cref{thm:irreducible-dim-real} is parameterized by the map
\begin{align}\label{eq:paramrealequi}
    \mu_{\mathbf{r},n} \, \coloneqq \, \mu_{r_{1,1},d_1} \times \mu_{r_{2,1},d_2} \times \prod \limits_{l\geq 3}  \prod \limits_{\substack{m\,\in \,(\ZZ/l\ZZ)^{\times}, \\ \frac{1}{2}< \frac{m}{l} < 1}} \mathcal{R}\mu_{r_{l,m},d_l}^{\CC}\, ,
    \end{align}
where $Q_\sigma$ is the realization base change of~$P_\sigma$. 
\end{proposition}
\begin{proof}
   After the base change $Q_\sigma$, every matrix in the real irreducible component corresponding to the integer solution $\mathbf{r}$ has the block diagonal structure~\eqref{eq:irrcompequi}. The blocks associated with $l\ge 3$ and $\frac{1}{2}< \frac{m}{l} < 1$ are the realization of $\mathcal{M}_{r_{l,m},\d_l\times d_l}({\CC})$, so they admit a parameterization via  $\mathcal{R}\mu_{r_{l,m},d_l}^{\CC}$.  Therefore, the real component $(\cE_{r,n\times n}^{\sigma,\mathbf{r}})^{\sim Q_{\sigma}}$ is parameterized by~$\mu_{\mathbf{r},n}$. 
\end{proof}

As a direct consequence of the parameterization in~\eqref{eq:paramrealequi}, one deduces a sparsity of the encoder and decoder into which the matrices of the respective irreducible component decompose into. We also obtain a weight-sharing property arising from the realization matrix blocks: the entries on the diagonal of each such matrix are equal, and the ones on the anti-diagonal differ only by a sign.
We now demonstrate these findings in our running example of rotation-invariance.
\begin{example}\label{ex:weightequi}
    We revisit \Cref{ex:real-comp-3by3}. The real irreducible component $(\cE_{3,9\times 9}^{\sigma,\mathbf{r}})^{\sim Q_{\sigma}}$ with $\mathbf{r}=(1,0,1)$ is
    \begin{align}
    \label{eq:(1,0,1)-para}
        \cM_{1,3\times 3}(\RR) \,\times \,  \cM_{0,2\times 2}(\RR) \,\times\, \mathcal{R}\left(\cM_{1,2\times2}(\CC)\right) \, .
    \end{align}
    By \Cref{prop:real-param-equi}, this component is parameterized by $\mu_{\mathbf{r},3}= \mu_{1,3}\times \mu_{0,2} \times \cR \mu_{1,2}^{\CC}$. Thus, every matrix in \eqref{eq:(1,0,1)-para} can be obtained as product of a $9\times 3$ and a $3 \times 9$ matrix of~the~form
\begin{align}\label{eq:matrixblocksequi}
   \begin{pmatrix} 
            *&*&*&0&0&0&0&0&0\\
            0&0&0&0&0&&&&\\
            0&0&0&0&0&\multicolumn{4}{c}{\smash{\raisebox{.5\normalbaselineskip}{$\cR{\begin{pmatrix}
            \star&
            \star
        \end{pmatrix}}$}}}
    \end{pmatrix}^\top \cdot 
    \begin{pmatrix} 
            *&*&*&0&0&0&0&0&0\\
            0&0&0&0&0&&&&\\
            0&0&0&0&0&\multicolumn{4}{c}{\smash{\raisebox{.5\normalbaselineskip}{$\cR{\begin{pmatrix}
            \star&
            \star
        \end{pmatrix}}$}}}
    \end{pmatrix},
    \end{align} 
    where $*$ and $\star$ represent arbitrary real and complex entries, respectively. The induced weight-sharing property on the encoder and decoder is visualized in \Cref{fig:weightsharequi}. 
\end{example}
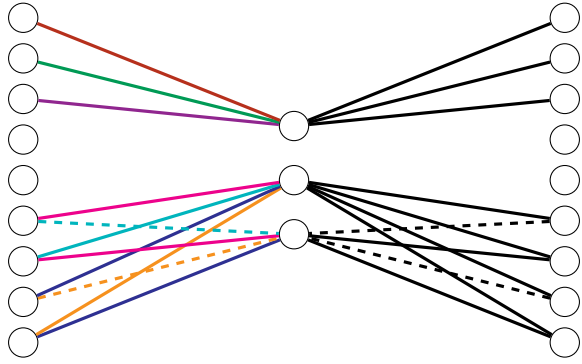
\begin{figure}[h]
\begin{tikzpicture}[scale=0.18]
 \node[shape=circle,draw=black] (A9) at (-5,12) {};
    \node[shape=circle,draw=black] (A8) at (-5,9) {};
	\node[shape=circle,draw=black] (A7) at (-5,6) {};
	\node[shape=circle,draw=black] (A6) at (-5,3) {};
	\node[shape=circle,draw=black] (A5) at (-5,0) {};
	\node[shape=circle,draw=black] (A4) at (-5,-3) {};
    \node[shape=circle,draw=black] (A3) at (-5,-6) {};
    \node[shape=circle,draw=black] (A2) at (-5,-9) {};
    \node[shape=circle,draw=black] (A1) at (-5,-12) {};
    \node[shape=circle,draw=black] (B3) at (15,4) {};
	\node[shape=circle,draw=black] (B2) at (15,0) {};
    \node[shape=circle,draw=black] (B1) at (15,-4) {};
    \node[shape=circle,draw=black] (C9) at (35,12) {};
    \node[shape=circle,draw=black] (C8) at (35,9) {};
    \node[shape=circle,draw=black] (C7) at (35,6) {};
    \node[shape=circle,draw=black] (C6) at (35,3) {};
    \node[shape=circle,draw=black] (C5) at (35,0) {};
    \node[shape=circle,draw=black] (C4) at (35,-3) {};
    \node[shape=circle,draw=black] (C3) at (35,-6) {};
    \node[shape=circle,draw=black] (C2) at (35,-9) {};
    \node[shape=circle,draw=black] (C1) at (35,-12) {};
	\path[very thick, Blue] (A1) edge node[left] {}  (B1); 
    \path[very thick, BurntOrange]  (A1) edge node[left] {} (B2);
    \path[very thick, dashed, BurntOrange]  (A2) edge node[left] {} (B1);
    \path[very thick, Blue] (A2) edge node[left] {}  (B2); 
    \path[very thick, Magenta]  (A3) edge node[left] {} (B1);
    \path[very thick, Aquamarine] (A3) edge node[left] {}  (B2); 
    \path[very thick, dashed, Aquamarine]  (A4) edge node[left] {} (B1);
    \path[very thick, Magenta] (A4) edge node[left] {}  (B2); 
    \path[very thick,  Plum]  (A7) edge node[left] {} (B3);
    \path[very thick, ForestGreen]  (A8) edge node[left] {} (B3);
    \path[very thick, BrickRed]  (A9) edge node[left] {} (B3);
    \path[very thick, black] (C1) edge node[left] {}  (B1); 
    \path[very thick, black]  (C1) edge node[left] {} (B2);
    \path[very thick, black, dashed]  (C2) edge node[left] {} (B1);
    \path[very thick, black] (C2) edge node[left] {}  (B2); 
    \path[very thick, black]  (C3) edge node[left] {} (B1);
    \path[very thick, black] (C3) edge node[left] {}  (B2); 
    \path[very thick, black, dashed]  (C4) edge node[left] {} (B1);
    \path[very thick, black] (C4) edge node[left] {}  (B2); 
    \path[very thick, black]  (C7) edge node[left] {} (B3);
    \path[very thick, black]  (C8) edge node[left] {} (B3);
    \path[very thick, black]  (C9) edge node[left] {} (B3);
\end{tikzpicture}
\caption{Weight-sharing of the encoder and decoder matrices from \Cref{ex:weightequi}. Edges of the same color share the same weight---and differ by sign, in case one of the edges is dashed. To avoid an overload of colors, we here visualized the weight-sharing for the encoder only; the decoder follows the same rules, but would require additional seven color shades. Due to the zero blocks in~\eqref{eq:matrixblocksequi}, the $4$th and $5$th input and output neurons are inactive.}
\label{fig:weightsharequi}
\end{figure}

A parameterization of  $\cE_{r,n\times n}^{\sigma,\mathbf{r}}$ is obtained by simply composing the autoencoders described in this section with the fixed matrix $Q_\sigma$ on the left, and its inverse $Q_\sigma^\top$ on the right. However, the sparsity of the weights of the en- and  decoder is easiest observed via the block diagonal form~\eqref{eq:irrcompequi}. For this reason, we formulated \Cref{prop:real-param-equi} for $(\cE_{r,n\times n}^{\sigma,\mathbf{r}})^{\sim Q_{\sigma}}$.

\subsection{Induced filtration of $\cM_{r,n\times n}$}
Let $\sigma \in \cS_n$. Whenever a matrix is equivariant under $\sigma$, then it is also equivariant under any power of~$\sigma$. Hence, $\cE_{n\times n}^{\sigma^k}\subset \cE_{n\times n}^{\sigma^{l\cdot k}}$ for all $k,l\in \NN$.  Therefore, any $\sigma\in \cS_n$ gives rise to an increasing filtration $\cE_{n\times n}^{\sigma^\bullet}$ of $\cM$.
This filtration is finite: every permutation has a finite order, hence $\cE^{\sigma^l}=\cE_{n\times n}^{\operatorname{id}}=\cM_{n \times n}$ for $l=\ord(\sigma)$. By intersecting with $ \cM_{r,n\times n}$, we obtain analogous statements for $\cE_{r,n\times n}^{\sigma^\bullet}$.

\subsection{Example: Equivariance for non-cyclic groups}\label{sec:exnoncyclic}
We here revisit equivariance for $3\times 3$ pictures. Characterizing equivariance for non-cyclic permutation groups is more complicated than the cyclic case. As a case study, we impose equivariance both under rotation and under reflection, i.e., we consider the group $G=\langle \sigma, \chi \rangle$ generated by the clock-wise rotation by $90$ degrees~$\sigma$ as in \eqref{eq:rot}, and the reflection
\begin{align}\label{eq:reflection}
\chi \colon \ \begin{pmatrix}
 a_{11} & a_{12} & a_{13}\\
a_{21} & a_{22} & a_{23}\\
a_{31} & a_{32} &  a_{33}
\end{pmatrix}  \mapsto  \begin{pmatrix}
 a_{13} & a_{12} & a_{11}\\
a_{23} & a_{22} & a_{21}\\
a_{33} & a_{32} &  a_{31}
\end{pmatrix} .
\end{align}
We will again identify $\RR^{3\times 3}\cong\RR^9$ via 
\begin{align}
\begin{pmatrix}
a_{11} & a_{12} & a_{13}\\
a_{21} & a_{22} & a_{23}\\
a_{31} & a_{32} &  a_{33}
\end{pmatrix} \mapsto 
\begin{pmatrix}
 a_{11} & a_{13} & a_{33} & a_{31} & a_{12} & a_{23} & a_{32}& a_{21} & a_{22}
\end{pmatrix}^\top.
\end{align}
Then $\chi(A)$ is represented by the vector 
$(a_{13} \ a_{11} \ a_{31} \ a_{33} \ a_{12} \ a_{21} \ a_{32} \ a_{23}\ a_{22})^\top.$
Under this identification, the reflection is $\chi=(1 \,2)(3\,4)(6 \,8)\in \cS_9$, and we denote its representing matrix by~$P_{\chi}$.
Hence, equivariance of a matrix $M=(m_{ij})_{i,j}\in \cM_{9\times 9}$ under both $\sigma$ and $\chi$, i.e., $MP_\sigma =P_\sigma M$ and $MP_\chi=P_\chi M$, implies that $M$ has to be of the form 
\begin{align}
M \,=\,
\begin{pmatrix}
\begin{array}{c|c|c}
\begin{matrix}
C_4(\alpha_1,\alpha_2,\alpha_3,\alpha_2)
\end{matrix} &
\begin{matrix}
C_4(\beta_1,\beta_2,\beta_2,\beta_1)
\end{matrix} &
\begin{matrix}
\varepsilon_3 \cdot \mathbbm{1}
\end{matrix} \\ \hline 
\begin{matrix}
C_4(\gamma_1,\gamma_1,\gamma_3,\gamma_3)
\end{matrix} &
\begin{matrix}
C_4(\delta_1,\delta_2,\delta_3,\delta_2)
\end{matrix} &
\begin{matrix}
\varepsilon_4 \cdot \mathbbm{1}
\end{matrix}\\ \hline 
\begin{matrix}
\varepsilon_1 \cdot \mathbbm{1}
\end{matrix}  &
\begin{matrix}
\varepsilon_2 \cdot \mathbbm{1}
\end{matrix} &
\begin{matrix}
\varepsilon_5
\end{matrix}
 \end{array}
\end{pmatrix} .
\end{align}
Therefore, $\dim(\cE_{9\times 9}^G)= 81-66=2\cdot 3+2\cdot 2+5\cdot 1=15$. In comparison to matrices that are required to be equivariant under the rotation~$\sigma$ only (see~\eqref{eq:matrotequi}), the entries $\alpha_4$, $\beta_3$, $\beta_4$, $\gamma_2$, $\gamma_4$, and $\delta_4$ can no longer be chosen freely, which drops the dimension by~$6$. 

\medskip
Let us add the action of another permutation on $3\times 3$ pictures, namely shifting each row by one to the right, i.e., for $i=1,2,3,$ $a_{i,j}\mapsto a_{i,j+1}$ for $j=1,2$, and $a_{i,3}\mapsto a_{i,1} $. In the choice from above, the shift corresponds to the permutation $(1 \, 5 \,2)(3 \,4 \,7)(6 \,8\, 9)\in \cS_9$.
All $9\times 9$ matrices $M$ that are equivariant under rotation, reflection, and shift, are of the following form, with only $3$ degrees $\alpha_1,\alpha_2,\alpha_3$ of freedom:
\begin{align}
M\,=\,\begin{pmatrix}
\begin{array}{c|c|c}
\begin{matrix}
C_4(\alpha_1,\alpha_2,\alpha_3,\alpha_2)
 \end{matrix}
 & \begin{matrix}
C_4(\alpha_2,\alpha_3,\alpha_3,\alpha_2)
 \end{matrix}
 &
 \begin{matrix}
 \alpha_3 \cdot \mathbbm{1}
 \end{matrix} \\ \hline
 \begin{matrix}
 C_4(\alpha_2,\alpha_2,\alpha_3,\alpha_3)
  \end{matrix} &
  \begin{matrix}
  C_4(\alpha_1,\alpha_3,\alpha_2,\alpha_3)
  \end{matrix} &
  \begin{matrix}
  \alpha_2 \cdot \mathbbm{1}
  \end{matrix}
 \\ \hline
  \begin{matrix}
\alpha_3 \cdot  \mathbbm{1}
 \end{matrix} & \begin{matrix}
 \alpha_2 \cdot \mathbbm{1}
\end{matrix} &\begin{matrix} \alpha_1 \end{matrix}
\end{array}
\end{pmatrix}.
\end{align}
\noindent To understand the general behavior, one will need to engage in combinatorial tailoring.

\section{Experiments}\label{sec:experiments}
We apply our findings to train various linear autoencoders on the dataset MNIST~\cite{MNIST}, a widely used benchmark in machine learning. Our implementations in {\tt Python}~\cite{Python} are made available at \href{https://github.com/vahidshahverdi/Equivariant}{\tt https://github.com/vahidshahverdi/Equivariant}. 
MNIST comprises $60{,}000$ training and $10{,}000$ test black-and-white images of handwritten digits, each with a size of $28 \times 28$ pixels. Utilizing MNIST images, we introduce random horizontal shifts of up to six pixels. Some representative images from the test dataset are shown in Figure~\ref{fig:MNIST}. 
\begin{figure}
    \centering
   \includegraphics[scale=0.9]{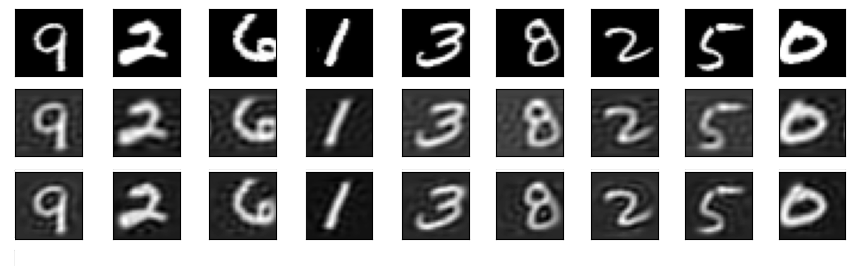}
    \caption{{\em Top row:} Nine samples from the MNIST~\cite{MNIST} test dataset, shifted  horizontally randomly by up to six pixels.
    {\em Middle row:} Output of a linear equivariant autoencoder designed to be equivariant under horizontal translations. The network architecture is determined by the integer vector $\mathbf{r}$, as described in \Cref{eq:order_r}.
    {\em Bottom row:} Output of a dense linear autoencoder with bottleneck~$r=99$ and no equivariance imposed.
    }
    \label{fig:MNIST}
\end{figure}
The task at hand is to design an autoencoder that is equivariant under horizontal translations. To achieve this, we first consider the permutation $\sigma \in \mathcal{S}_{784}$, representing a horizontal shift by one pixel. Consequently, we have $28$ disjoint cycles of size $28$, one for each row of images. The input data matrix~$X$ is a real matrix of size $784 \times 60{,}000$, where each column of $X$ represents the row-wise vectorization of the shifted images. It is important to note that, due to the structure of images, $X$ does not yield a full-rank matrix; in fact, 
its rank is $397 < 784$. We choose $r=99$, and our goal is to find a matrix  $M^{\sim Q_{\sigma}} = \bigoplus_{i=0}^{14} M_{i}$ in $(\mathcal{E}_{99,784 \times 784}^{\sigma})^{\sim Q_{\sigma}}$ such that $M$ minimizes $\lVert MX-X \rVert_F^2 \, .$ In here, $M_0$ and $M_{14}$ are matrices in $\cM_{r_0,28 \times 28}$ and $\cM_{r_{14},28 \times 28}$, respectively, and $M_i \in \cR(\cM_{r_i,28 \times 28}(\CC))$ for $i=1,\ldots, 13$, where $r_0 + 2r_1+\cdots +2r_{13}+r_{14}=99$. For increased readability, we write $r_i$ instead of $r_{l,m}$; see~\Cref{eq:order_r} for the precise matching of the indices.
As explained in \Cref{sec:parameterizingEquiv}, any network can parameterize only one of the real irreducible components $\mathcal{E}_{r,784 \times 784}^{\sigma,\mathbf{r}}$ of $\cE_{r,784 \times 784}^{\sigma}$. Referring to \Cref{thm:irreducible-dim}, we find that the number of irreducible components of $\mathcal{E}_{r,784 \times 784}^{\sigma}$ is
\begin{align}\label{eq:numberirrcompg}
   72{,}425{,}986{,}088{,}826 \, .
\end{align}
  Among these numerous components, we empirically observed that the component corresponding to the following integer vector $\mathbf{r}$ yields a reasonable loss for $r=99$: 
\begin{align}
\label{eq:order_r}
    \mathbf{r} &\,=\, (r_0,r_1,r_2,r_3,r_4,r_5,r_6,r_7,r_8,r_9,r_{10},r_{11},r_{12},r_{13},r_{14}) \nonumber \\
    &\,=\, (r_{1,1}, r_{28,27}, r_{14,13}, r_{28,25}, r_{7,6}, r_{28,23}, r_{14,11}, r_{4,3}, r_{7,5}, r_{28,19}, r_{14,9}, r_{28,17}, r_{7,4}, r_{28,15}, r_{2,1}) \\ \nonumber
    &\,=\, (13, 10, 9, 8, 7, 5, 3, 1, 0, 0, 0, 0, 0, 0, 0) \, .
\end{align}

The choice of the $r_{l,m}$ depends entirely on the structure of the data. Our observation reveals that, in each column of~$X$, the energy is concentrated in low frequencies, as is illustrated in \Cref{fig:covmatrix}. Consequently, the blocks corresponding to eigenvalues $\lambda_{l,m}= e^{\frac{2\pi m}{l}}$ with phases close to the zero angle contain more information. Thus, to effectively encode this information, it might be required to put higher ranks on blocks $M_i$ for $i$ close to~$0$.  \Cref{fig:energy} is a visual representation of this energy distribution for each Fourier mode. In \Cref{fig:MNIST}, we showcase the output for nine samples using our equivariant autoencoder with architecture~$\mathbf{r}$, and compare it to a linear autoencoder with $r=99$, {without  equivariance imposed}.

\smallskip
\begin{figure}[h]
    \centering
   \includegraphics[scale=0.8]{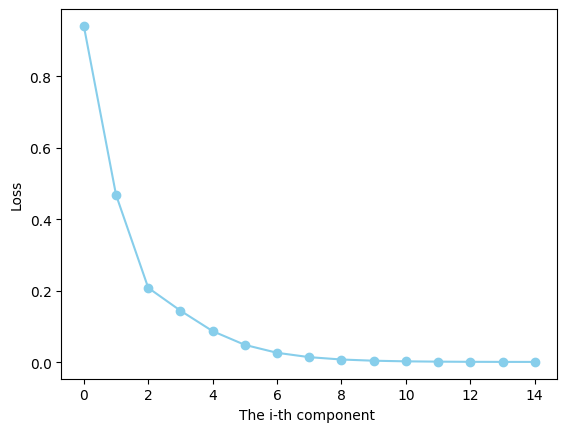}
    \caption{The error incurred by the block $M_i$, $i =0, \ldots, 14$,  when setting $\rank(M_i)=0 .$}
    \label{fig:energy}
\end{figure} 

\begin{figure}[H]
\centering
   \includegraphics[scale=0.86]{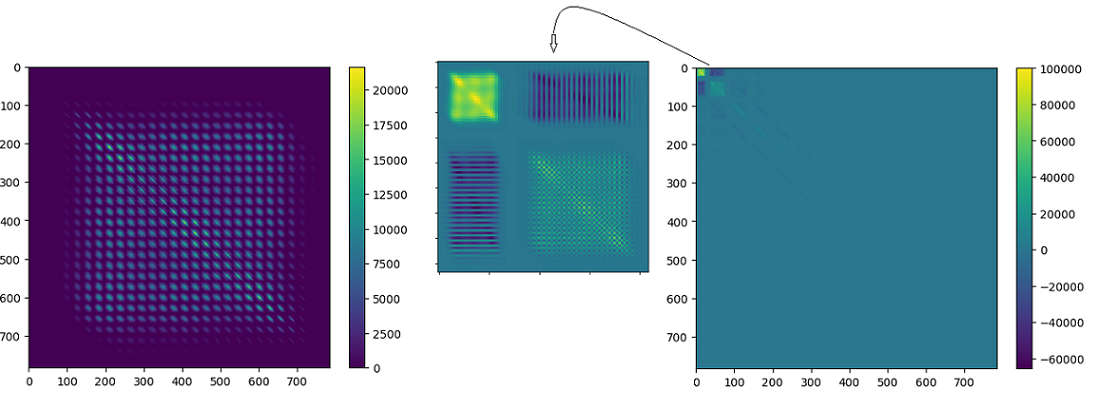}
    \caption{On the left, $XX^\top$ is visualized as an image, where $X$ represents our input data. On the right, $(XX^\top)^{\sim Q_{\sigma}}$ is depicted, 
    {with $Q_{\sigma}$ as defined in \Cref{sec:real-irreduce}.} 
Analyzing the plot on the right hand side, we note that the majority of the signal's energy is concentrated in low~frequencies.
}
\label{fig:covmatrix}
\end{figure} 

The significance of a proper choice of~$\mathbf{r}$ is illustrated in  \Cref{fig:high-mid-energy}. In this figure, we present the outputs of a handwritten digit~``$6$'' under two different equivariant architectures: first, using $r_{l,m}=7$ for every $l$ and $m$, and second, by excluding the first seven blocks following the order in \eqref{eq:order_r}, and the remaining blocks having full rank. Despite both scenarios having a total rank $r$ greater than~$99$, the outcomes are notably inferior.

\vspace*{-1mm}
\begin{figure}[H]
    \centering
\includegraphics[scale=1.9]{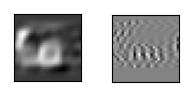}
\vspace*{-2mm}
    \caption{The outputs of linear equivariant autoencoders with suboptimal architectures for a handwritten digit ``6". {\em On the left:} an equivariant autoencoder with equal rank distributed among the blocks, setting all $r_{l,m}=7$. {\em On the right:} High-pass equivariant autoencoder, excluding the first seven blocks in the order of~\eqref{eq:order_r}.}
    \label{fig:high-mid-energy}
\end{figure} 
{From the comparison presented in \Cref{tab:loss_comparison}, one reads that a dense linear autoencoder without imposed equivariance---on average---outperforms equivariant architectures in terms of the mean squared loss.} However, achieving this superior performance demands a substantial parameter count, totaling $2\cdot 99\cdot 784=155{,}232$.  In contrast, our proposed equivariant autoencoder, defined by the architecture in \eqref{eq:order_r}, requires a more efficient parameter count of $2\cdot (28\cdot 13 + 2 \cdot  28\cdot (10+9+8+7+5+3+1))=5{,}544$. It is worth mentioning that, since we empirically chose one of the numerous irreducible components of $\cE_{99,784\times 784}^\sigma$, our proposed equivariant autoencoder may not represent the optimal choice among all possible linear equivariant architectures with~$r=99$. 
\begin{table}[H]
    \centering
    \small 
    \footnotesize
    \begin{tabular}{lSSSS}
        \hline
        & {Equivariant architecture \eqref{eq:order_r}} & {equal-rank equivariant} & {high-pass equivariant} & {non-equivariant} \\
        \hline
        Loss & 0.0082 & 0.0206 & 0.1063 & {\bf 0.0057} \\
        \hline
    \end{tabular}
    \caption{Comparison of average square loss values per pixel between linear equivariant and non-equivariant autoencoders on the MNIST test dataset.
    The equal-rank equivariant architecture is achieved by setting all $r_{l,m}=7$, while the high-pass equivariant architecture is obtained by letting the first seven blocks have rank $0$, and the remaining blocks have full rank. The non-equivariant network is trained with~$r=99$.}
    \label{tab:loss_comparison}
\end{table}
The non-equivariant linear autoencoder with a bottleneck size of $r=99$ exhibits partial equivariance under horizontal shifts, as depicted in \Cref{fig:W_trained-nonequi}. This behavior arises from the fact that for larger shifts, handwritten digits could split into two parts (scenarios which are not---or only barely---present in our modified dataset). 
This observation might be a partial explanation for the superior performance of the dense linear autoencoder compared to our equivariant architecture---the latter, however, is more efficient.

\begin{figure}[H]
   \hspace*{-4mm}
   \includegraphics[scale=0.9]{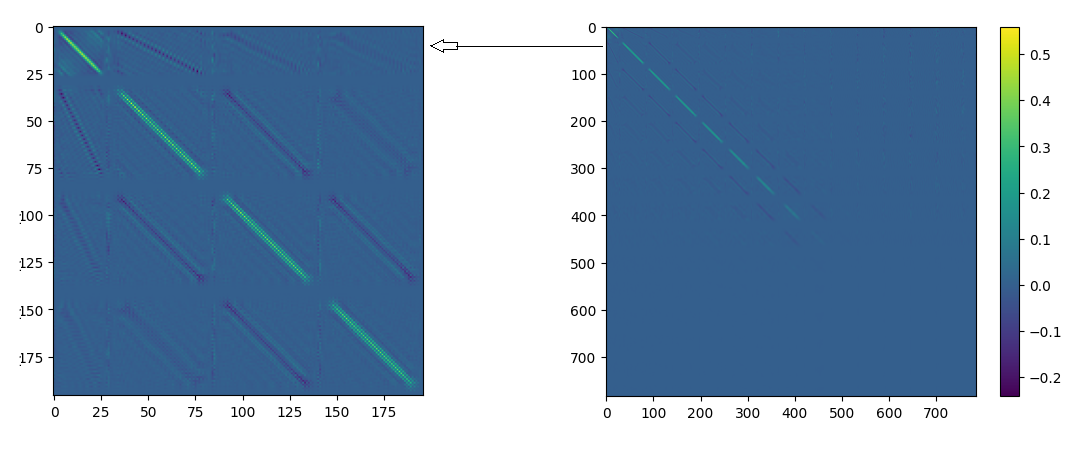}
    \caption{The plot on the right represents the trained dense linear autoencoder with bottleneck $r=99$ after the base change by $Q^{\sim \sigma}$. The one on the left is a magnified image of the plot on the right. It is evident that this final matrix is only partially equivariant under horizontal translations.}
    \label{fig:W_trained-nonequi}
\end{figure} 

\section{Conclusion and outlook}\label{sec:outlook}
We investigated linear neural networks through the lens of algebraic geometry, with an emphasis on linear autoencoders. Their function spaces are determinantal varieties $\cM_{r,n \times n}$ in a natural way. We considered permutation groups $G$ and fully characterized the elements of the function space which are invariant under the action of~$G$. They form an irreducible algebraic variety $\cI_{r,n\times n}^G\subset \cM_{r,n\times n}$ for which we computed the dimension, singular points, and degree.
We showed that the squared-error loss can be minimized on that variety by an explicit calculation using the Eckart--Young theorem.
We proved that all $G$-invariant functions can be parameterized by a linear autoencoder, and we derived implications for the design of such an autoencoder, namely a dimensional constraint on the middle layer and weight sharing in the encoder.  
For equivariance, we treated cyclic subgroups $G=\langle \sigma \rangle$ of permutation groups. Also in this case, the resulting part of the function space is an algebraic variety $\cE_{r,n \times n}^\sigma \subset \cM_{r,n\times n}$. Typically, this variety has several irreducible components; we determined them both over $\CC$ and over $\RR$. We computed their dimension, singular locus, and degree (the latter only for the complex components). Since $\cE_{r,n \times n}^\sigma$ is reducible, no linear neural network can parameterize all of $\cE_{r,n \times n}^\sigma$. However,  we provided a parameterization of each real irreducible component via a sparse autoencoder with the same weight sharing on its en- and decoder. We also explained a simple algorithm that reduces squared-error loss minimization on each real component to applying the Eckart--Young theorem multiple~times. 

\smallskip

Our technique for linear equivariant autoencoders—using rank constraints combined with equivariance—can be generalized to create a broader class of equivariant autoencoders. By representing each layer as a square matrix with constrained rank, we control the number of parameters while maintaining equivariance. The key advantage is the simplified analysis, as the same group representation is applied to both the input and output of each layer. This consistency across layers reduces the complexity, and provides a clear path for designing and analyzing general equivariant architectures.

\smallskip 

{To showcase our results, we trained several autoencoders on the MNIST dataset. 
For a bottleneck rank of $r=99$, the space of linear functions that are equivariant under horizontal shifts has the gigantic number $72{,}425{,}986{,}088{,}826$ of real irreducible components. 
We carefully chose one of the components and compared the outcome of an equivariant autoencoder parameterizing that component with a linear autoencoder without imposed equivariance. The latter did achieve a lower loss, requires however significantly more parameters. We also give a partial explanation of the superior performance of a general linear network; it arises from the nature of the considered dataset, which might be partially equivariant only.} 

\smallskip 
The generalization to non-cyclic groups is more intricate than for invariance. We plan to tackle this problem in follow-up~work.
One should also address groups other than permutation groups, such as  non-discrete groups. Another natural step to take is to generalize the network architecture to a bigger number of layers as well as to allowing non-trivial activation functions, such as~ReLU. For the latter, we expect that tropical expertize will be helpful to study the resulting geometry of the function space. Having the geometry of the function spaces understood, one should also investigate the types of critical points during training processes and how they compare to networks without imposed equi- or invariance.

\smallskip

Finally, a future goal is to employ the insights from this article to graph neural networks (GNNs).
While previous works on GNNs, such as \cite{maron2018invariant}, mainly focused on equivariance and invariance under the full permutation group, our study extends these ideas by considering subgroups of the permutation group. This generalization is important for applications where the symmetry group is more restricted, enabling the design of architectures for specific tasks.

\pagebreak
\noindent{\bf Acknowledgments.}
We thank Joakim And\'{e}n and Luca Sodomaco for insightful discussions on our experiments and on ED degrees, respectively. We also thank the anonymous reviewers whose suggestions helped us to improve this article. KK and ALS were partially supported by the Wallenberg~AI, Autonomous Systems and Software Program~(WASP) funded by the Knut and Alice Wallenberg Foundation. 

\addcontentsline{toc}{section}{References}
{
\normalsize
\bibliographystyle{abbrv} 
}
\end{document}